\pdftrailerid{}
\PassOptionsToPackage{table}{xcolor}
\documentclass[letterpaper,10pt]{routerpreprint}
\usepackage[utf8]{inputenc}
\usepackage{amssymb,amsthm,mathtools,array,tabularx}
\usepackage{algorithm,algpseudocode,needspace,xurl,titletoc,fontawesome5}
\tcbuselibrary{breakable}
\setcitestyle{authoryear,round,citesep={;},aysep={,},yysep={;}}
\newcolumntype{Y}{>{\raggedright\arraybackslash}X}
\definecolor{trmteal}{HTML}{207D80}
\definecolor{trmlight}{HTML}{EDF7F6}
\definecolor{mailMa}{HTML}{9B303F}
\definecolor{mailMi}{HTML}{2959A7}
\newcommand{\MailIcon}[1]{\textcolor{#1}{\faIcon[regular]{envelope}}}
\newcommand{\trm}{\mbox{\textup{T-Router}}}
\newcommand{\R}{\mathbb{R}}
\newcommand{\rms}{\operatorname{RMS}}
\newcommand{\sg}{\operatorname{sg}}
\newcommand{\softmax}{\operatorname{softmax}}
\newcommand{\E}{\mathbb{E}}

\newtheorem{proposition}{Proposition}
\title{T-Router: Learning Thalamic Routing for Reasoning\\with Parameter-Efficient Reinforcement Learning}
\author[1,*,\MailIcon{mailMa}]{Liuxian Ma}
\author[2,*]{Jiale Dai}
\author[3]{Jiaqi Li}
\author[1,\ensuremath{\dagger},\MailIcon{mailMi}]{Lu Mi}
\affiliation[1]{College of Artificial Intelligence, Tsinghua University}
\affiliation[2]{State Key Laboratory of General Artificial Intelligence, School of Intelligence Science and Technology, Peking University}
\affiliation[3]{Beijing Institute for General Artificial Intelligence}
\contribution[*]{Strictly equal contribution; either author order is equally valid.}
\contribution[\ensuremath{\dagger}]{Corresponding author.}
\metadata[Keywords]{parameter-efficient reinforcement learning, large language models, mathematical reasoning, cross-layer routing, computation reuse, thalamic routing}
\metadata[Contact]{\MailIcon{mailMa}\,\href{mailto:maliuxian03@gmail.com}{\textcolor{black}{\texttt{maliuxian03@gmail.com}}}\quad\MailIcon{mailMi}\,\href{mailto:milu@mail.tsinghua.edu.cn}{\textcolor{black}{\texttt{milu@mail.tsinghua.edu.cn}}}}
\hypersetup{pdftitle={T-Router: Learning Thalamic Routing for Reasoning with Parameter-Efficient Reinforcement Learning},pdfauthor={Liuxian Ma, Jiale Dai, Jiaqi Li, Lu Mi},pdfsubject={Parameter-efficient reasoning reinforcement learning},pdfkeywords={parameter-efficient reinforcement learning, large language models, mathematical reasoning, cross-layer routing, computation reuse, thalamic routing}}

\definecolor{paperhead}{HTML}{E7EDF6}
\definecolor{paperrule}{HTML}{A6B2C3}
\definecolor{paperink}{HTML}{33395B}
\definecolor{paperback}{HTML}{F7F8FB}
\colorlet{trmlight}{paperhead!65!white}
\newtcolorbox{ThuPrompt}{breakable,sharp corners,colback=paperback,colframe=paperrule,
  boxrule=.35pt,left=7pt,right=7pt,top=5pt,bottom=5pt,
  before skip=8pt,after skip=8pt}
\newtcolorbox{ThuStatement}{breakable,sharp corners,colback=paperhead!38!white,colframe=paperrule,
  boxrule=.35pt,left=7pt,right=7pt,top=0pt,bottom=0pt,
  before skip=8pt,after skip=8pt}
\newcommand{\ThuTableCaptionGap}{4pt}
\AddToHook{env/table/begin}{\setlength{\belowcaptionskip}{\ThuTableCaptionGap}}
\newcommand{\ThuAppendixLayout}{\renewcommand{\arraystretch}{1.00}
  \renewcommand{\ThuTableCaptionGap}{2pt}\setlength{\parskip}{4.5pt}\setlength{\textfloatsep}{18pt plus 2pt minus 2pt}\setlength{\floatsep}{12pt plus 2pt minus 2pt}}

\newcommand{\zerodash}[1]{\textemdash}

\newcommand{\stat}[2]{\ensuremath{#1\,{\scriptstyle\pm #2}}}
\newcommand{\beststat}[2]{\ensuremath{\mathbf{#1}\,{\scriptstyle\pm #2}}}
\newcommand{\secondstat}[2]{\ensuremath{\underline{#1}\,{\scriptstyle\pm #2}}}

\colorlet{trmnavink}{black}
\colorlet{trmnavmuted}{black}
\colorlet{trmnavrule}{paperrule}
\titlecontents{section}[1.6em]
  {\addvspace{6pt}\fontsize{10}{12}\selectfont\sffamily\bfseries\color{black}}
  {\contentslabel{1.6em}}
  {}{\hspace{.5em}\hfill\contentspage}
\titlecontents{subsection}[2.9em]
  {\addvspace{.6pt}\fontsize{9.2}{11.2}\selectfont\color{black}}
  {\contentslabel{2.5em}}
  {}{\hspace{.5em}\titlerule*[4pt]{.}\contentspage}
\newcommand{\TRMAppendixFront}{
  \phantomsection\label{app:contents}
  \pdfbookmark[0]{Appendix contents}{trm.appendix.contents}
  \startcontents[trmappendix]
  \startcontents[trmspecification]
  \startcontents[trmresults]\stopcontents[trmresults]
  \begingroup\setlength{\parskip}{0pt}
  \noindent{\color{metablue}\rule{\linewidth}{1.2pt}}\par
  \vspace{8pt}
  \noindent{\huge\sffamily Appendix}\par
  \vspace{5pt}
  {\small T-Router: Learning Thalamic Routing for Reasoning\\with Parameter-Efficient Reinforcement Learning\par}
  \vspace{7pt}
  \noindent{\color{trmnavrule}\rule{\linewidth}{.4pt}}\par
  \vspace{7pt}
  {\small Computational specification, formal analysis, and the supporting experimental results.\par}
  \vspace{10pt}
  \noindent\begin{minipage}[t]{.48\linewidth}
    {\small\sffamily\bfseries Specification and analysis}\par
    \vspace{3pt}
    {\color{trmnavrule}\rule{\linewidth}{.4pt}}\par
    \printcontents[trmspecification]{}{1}{\setcounter{tocdepth}{2}}
  \end{minipage}\hfill
  \begin{minipage}[t]{.48\linewidth}
    {\small\sffamily\bfseries Results and interpretation}\par
    \vspace{3pt}
    {\color{trmnavrule}\rule{\linewidth}{.4pt}}\par
    \printcontents[trmresults]{}{1}{\setcounter{tocdepth}{2}}
  \end{minipage}\par
  \vspace{12pt}
  \noindent{\color{trmnavrule}\rule{\linewidth}{.35pt}}\par
  {\fontsize{8.2}{10}\selectfont
    \textbf{Reading routes}\enspace
    \hyperref[app:implementation]{Implementation A}\enspace /\enspace
    \hyperref[app:formal]{Theory B}\enspace /\enspace
    \hyperref[app:results-atlas]{Results C}\enspace /\enspace
    \hyperref[app:design]{Interpretation D}\par}
  \endgroup
  \clearpage
}

\abstract{Parameter-efficient reinforcement learning aims to improve reasoning with a compact trainable interface to a pretrained model. We introduce the \emph{Thalamic Router} (\trm), which concentrates adaptation on the reuse of completed computations. A compressed, addressable bank preserves block changes; a depth-recurrent controller conditions their selection and relative-scale writeback. This coupling gives thalamic context-dependent routing a concrete computational form: learn which earlier contributions a receiving layer uses, and with what influence. Correctness rewards train the interface while preserving backbone parameters and layer order. On an 8.95B-parameter backbone, \trm\ allocates 41.73M parameters---0.466\% of the backbone---and achieves $83.64\pm1.16$ MathAvg after GSM8K RL, compared with $73.79\pm1.83$ for full-parameter GRPO across three evaluation rounds. At a comparable parameter budget and with matched retries, it exceeds LoRA's $77.28\pm1.95$ MathAvg, improving all three task families and raising mean AIME accuracy from 48.33 to 60.56. Capacity-controlled comparisons favor addressable block changes and recurrent context; separate search training extends the interface to tool-mediated reasoning. These results establish controlled computation reuse as an effective route to parameter-efficient reasoning reinforcement learning.
}
\begin{document}
\maketitle
\section{Introduction}
Reinforcement learning (RL) improves mathematical reasoning by rewarding successful solutions \citep{shao2024deepseekmath}. Making this adaptation efficient requires placing trainable capacity where reward feedback can improve reasoning while preserving the bulk of pretrained parameters \citep{sidahmed2024perlhf}. A model already contains useful transformations; an additional opportunity is to learn how their intermediate contributions are organized. Our starting point is therefore \emph{adaptation efficiency}: train a compact \emph{Thalamic Router} (\trm) for \emph{controlled reuse of completed computations} around a frozen backbone.

\begin{figure}[t]
\centering
\includegraphics[width=\linewidth]{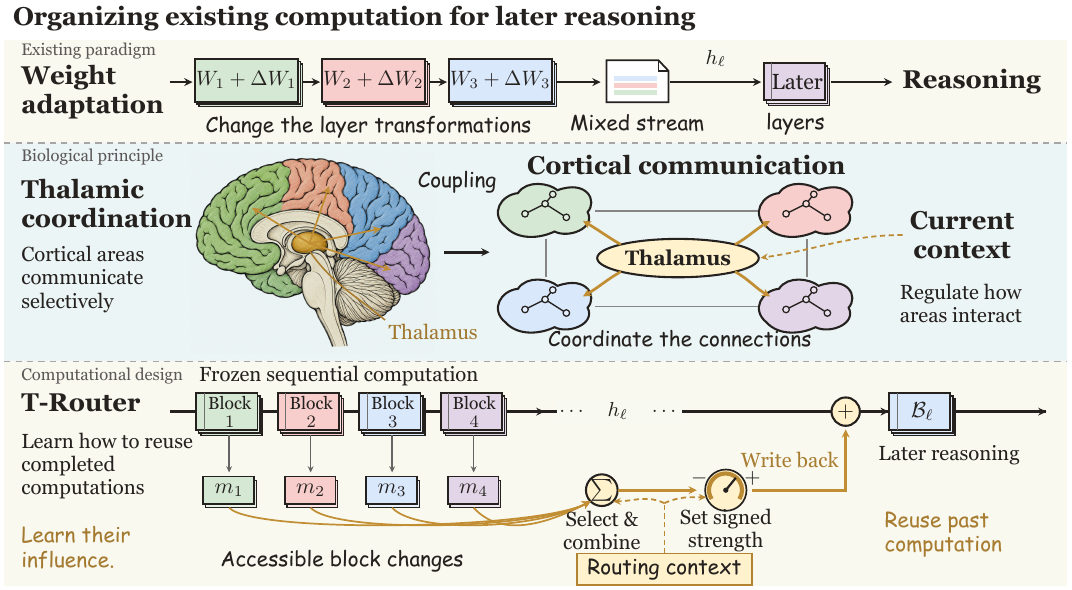}
\caption{\textbf{From adapting transformations to learning their communication.} Weight adaptation changes transformations (top). Thalamic circuits regulate cortical communication according to task context (middle). T-Router applies this routing principle across depth: select addressable block changes and regulate their influence before a later frozen layer executes (bottom).}
\label{fig:problem}
\end{figure}

LoRA and adapters place trainable capacity in weight updates or local transformations \citep{hu2022lora,houlsby2019adapters}. In the resulting residual stream, earlier block contributions continue to mix with subsequent updates. Keeping these contributions separately addressable exposes an additional learning target: which completed changes a later computation reuses, and with what influence. Learned depth connectivity already develops this direction through attention to prior computation, additive delta reuse, and frozen-backbone residual routing \citep{kimi2026attnres,luo2026delta,oldenburg2026mhcpeft}. The question for a compact RL interface is how to couple reusable content, the receiving computation's context, and writeback strength.

Thalamic circuits regulate cortical communication according to task demands. The pulvinar coordinates information transmission between cortical areas through attention-dependent synchronization \citep{saalmann2012pulvinar}; mediodorsal thalamic input amplifies functional prefrontal connectivity to sustain task representations \citep{schmitt2017thalamic}. We take these findings as a routing principle: use context to select information sources and regulate their influence on a receiving computation. In \trm, frozen decoder blocks perform computation, while the complete auxiliary pathway controls which completed block changes a later layer reuses and at what strength. Figure~\ref{fig:problem} makes this correspondence explicit.

\trm\ preserves compressed block changes in a source bank alongside the evolving residual. Later layers can recombine these records according to their current state and accumulated depth context. Controller slots $S$ accumulate that history; attention queried by the current hidden state reads a context vector $P$ from the updated slots. Together with the hidden state, $P$ conditions both source attention and a signed gate. Selected content is projected into the receiving residual and added before the frozen layer executes. The bank defines \emph{what is available}; the router and gate learn \emph{which mixture to reuse and at what strength}.

Correctness rewards train this communication interface end to end. Its defining feature is the coupling of \emph{addressable content}, \emph{recurrent control}, and \emph{relative influence}: the bank preserves reusable changes, depth context guides their selection, and a signed gate regulates the resulting writeback. RMS calibration expresses the gate's strength relative to the receiving residual. Every pretrained layer still executes in its original order. RL can therefore optimize how completed contributions serve a later computation while preserving the backbone's pretrained transformations.

The mathematics comparison demonstrates this efficiency at two adaptation budgets. With 41.73M parameters, 0.466\% of the 8.95B backbone, \trm\ reaches $83.64\pm1.16$ MathAvg versus $73.79\pm1.83$ for full-parameter GRPO across three evaluation rounds. At a compact budget, matched-retry LoRA uses 43.28M parameters and reaches $77.28\pm1.95$. \trm\ improves all three mathematical families, including 92.73 versus 87.93 on MATH-500 and 60.56 versus 48.33 mean AIME accuracy after GSM8K training. Capacity-controlled comparisons favor the full interface over hidden-state memory (67.46 MathAvg) and a feedforward controller replacement (72.80). Separate search training applies the same compact interface to tool-mediated reasoning.

Our contributions are: (i) a parameter-efficient RL interface coupling compressed, addressable block changes to independent depth-recurrent control and relative-scale writeback; (ii) stronger mathematical reasoning than full-parameter GRPO with 0.466\% of the backbone's parameter count, and than matched-retry LoRA at a comparable budget, with an application to agentic search; and (iii) capacity-controlled comparisons and a formal characterization of the content and control pathways. Together they establish controlled computation reuse as a concrete route to efficient RL adaptation.

\section{Related work}
\paragraph{Parameter-efficient reinforcement learning.}
Adapters and LoRA reduce the trainable footprint of task adaptation \citep{houlsby2019adapters,hu2022lora}; PE-RLHF applies parameter-efficient learning to reward modeling and reinforcement learning \citep{sidahmed2024perlhf}. GRPO uses group-relative rewards \citep{shao2024deepseekmath}; Section~\ref{sec:method} identifies the single-update policy gradient used to train the module. RO-GRPO shapes rewards using routing statistics in LoRA mixtures \citep{ma2026balancing}, while S-GRPO samples output-token positions for the training loss \citep{lee2025token}. DAPO oversamples rollout groups and filters all-correct or all-incorrect groups \citep{yu2025dapo}; \trm\ uses a bounded retry rule, analyzed in Appendix~\ref{app:formal}. \trm\ retains input-dependent cross-layer communication at inference; Appendix~\ref{app:implementation} details the parameter and execution costs.

\paragraph{Learning communication across depth.}
DenseNet enables feature reuse through dense concatenation of preceding feature maps \citep{huang2017densenet}. DenseFormer learns depth-weighted combinations \citep{pagliardini2024denseformer}; MUDDFormer makes cross-layer connections dynamic and stream-specific \citep{xiao2025muddformer}. Attention Residuals introduces attention over earlier computation, including a block variant \citep{kimi2026attnres}. Delta Attention Residuals adds attention-weighted sublayer or block changes to the residual \citep{luo2026delta}, providing the closest source-reuse precedent. Multi-Head Attention Residuals gives feature subspaces separate distributions over depth \citep{luo2026multihead}. Hyper-connection constraints and adaptive reassignment further expand depth connectivity \citep{xie2026mhc,zhang2026ancre}. In the frozen-backbone setting, mHC-based PEFT learns read/write routing over multiple residual streams and investigates learned versus identity residual mixing \citep{oldenburg2026mhcpeft}. These methods establish depth connectivity as an adaptation axis alongside local weight updates.

\trm\ targets parameter-efficient RL through a coupled content-and-control interface: compressed, origin-indexed changes supply content; a separate depth-recurrent controller supplies routing context; and calibrated signed writeback supplies relative influence. A current-state query reads $P$ from slots $S$ to condition both selection and strength. Correctness rewards train this compact interface while preserving pretrained transformations and layer order. Appendix~\ref{app:design} develops these computational distinctions.

\paragraph{Thalamic regulation of cortical communication.}
The pulvinar coordinates cortical information transmission through attention-dependent synchronization \citep{saalmann2012pulvinar}. Mediodorsal thalamic input amplifies functional prefrontal connectivity, sustaining rule representations \citep{schmitt2017thalamic}. These mechanisms motivate context-dependent control over the contributions of distributed computations. \trm\ realizes this principle through source selection and writeback strength, conditioned on the receiving residual and recurrent depth context. HippoRAG draws on hippocampal indexing theory for external knowledge retrieval \citep{gutierrez2024hipporag}.

\section{T-Router}\label{sec:method}

\trm\ concentrates trainable capacity in a compact communication path around a frozen decoder. A source bank stores compressed block changes, a recurrent controller summarizes depth history, and a query-dependent readout conditions source attention and writeback strength. The router and gate combine this signal with the current residual to form an intervention. All pretrained layers execute in order; correctness rewards train how earlier contributions are reused.

\begin{figure}[t]
\centering
\includegraphics[width=\linewidth]{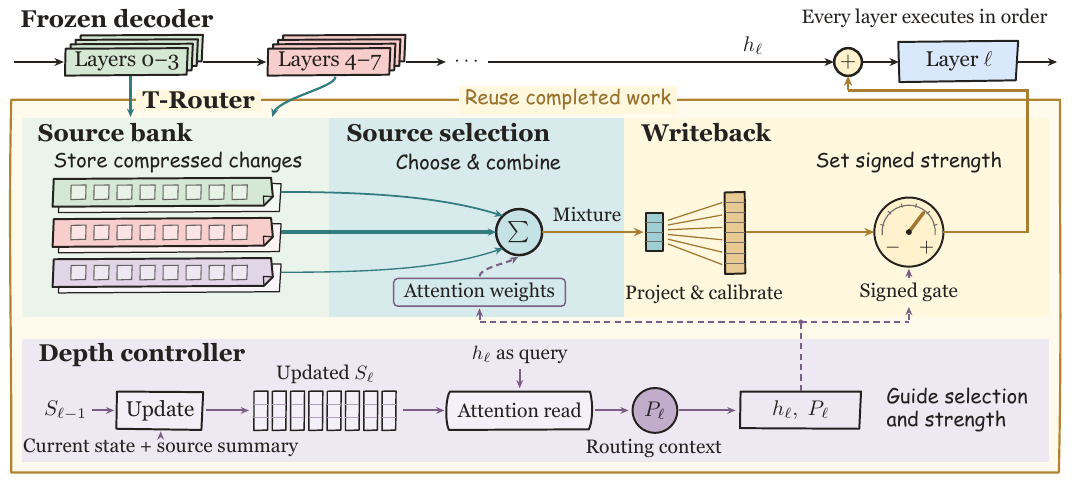}
\caption{\textbf{How T-Router reuses completed computation.} Completed blocks supply compressed changes to the source bank. The depth controller updates its state $S$, then reads routing context $P$ using the current hidden state. This context conditions which sources are mixed and how strongly the projected mixture is written back. The contribution is added before the receiving frozen layer; all layers execute in order. Teal and gold paths carry reusable content; dashed purple branches condition selection and strength. Equations~\ref{eq:controller-readout}--\ref{eq:writeback} specify the operations.}
\label{fig:framework}
\end{figure}

\Needspace{8\baselineskip}
\subsection{Recording changes across depth}
Let $H_\ell\in\R^{B\times T\times d}$ be the residual tensor before layer $\ell$, for batch size $B$ and sequence length $T$. With the batch index suppressed, $h_{\ell,t}\in\R^d$ denotes its token-$t$ vector; $R_\ell$ stacks the token interventions $R_{\ell,t}$. The complete frozen decoder layer $\mathcal B_\ell$ acts on the sequence tensor, including its causal token interactions and residual connections:
\begin{equation}
\widetilde H_\ell=H_\ell+R_\ell,\qquad
H_{\ell+1}=\mathcal B_\ell(\widetilde H_\ell).
\label{eq:execution}
\end{equation}
Writing $\mathcal B_\ell(X)=X+F_\ell(X)$ gives $H_{\ell+1}=H_\ell+R_\ell+F_\ell(H_\ell+R_\ell)$: the intervention enters both the residual connection and the frozen transformation. Training updates only the auxiliary module.

We group consecutive layers into blocks of size $s$. Upon completing block $b$, \trm\ compresses the change between its first actual input and final output:
\begin{equation}
D_{b,t}=h_{s(b+1),t}-\widetilde h_{sb,t},\qquad
m_{b,t}=C_bD_{b,t}\in\R^r.
\label{eq:memory}
\end{equation}
Here $C_b\in\R^{r\times d}$ is learned. This change is measured along the intervened forward pass: writebacks inside the block can contribute to $D_b$. Before layer $\ell$, only completed blocks
$\mathcal J_\ell=\{b\in\mathbb N_0:s(b+1)\leq\ell\}$ are available. Each record pairs $m_{b,t}$ with a learned source embedding $e_b^{\mathrm{src}}$. The bank stores changes at the same token position; it introduces no new attention over token positions.

We use $L=32$, $s=4$, and $r=256$. The first four layers receive zero writeback; later layers can read up to seven completed blocks from the same forward pass.

\subsection{From depth state to controller context}
The controller summarizes the token's computation across layers to condition source selection and writeback strength. Its state comprises $K$ slots $S_{\ell,t,k}\in\R^p$, initialized as $S_{-1,t,k}=S_k^0$ from a learned array. At every layer, the update combines the current residual, the mean $\bar m_{\ell,t}$ of completed records (zero for an empty bank), and the receiving-layer embedding $e_\ell$:
\begin{equation}
z_{\ell,t}=[A_hh_{\ell,t};A_m\bar m_{\ell,t};e_\ell],
\qquad [u_{\ell,t};v_{\ell,t}]=\operatorname{MLP}(z_{\ell,t}).
\label{eq:controller-input}
\end{equation}
Write attention $\omega$ distributes the shared proposal across the previous slots:
\begin{equation}
\begin{aligned}
\omega_{\ell,t,k}
&=\softmax_k\!\left(S_{\ell-1,t,k}^{\top}W_sz_{\ell,t}/\sqrt p\right),\\
S_{\ell,t,k}
&=\gamma S_{\ell-1,t,k}
+(1-\gamma)\omega_{\ell,t,k}\bigl[\sigma(v_{\ell,t})\odot\tanh(u_{\ell,t})\bigr].
\end{aligned}
\label{eq:controller}
\end{equation}
We define \emph{controller context} $P_{\ell,t}\in\R^p$ as an attention readout of these updated slots. The current residual supplies the query $Q_sh_{\ell,t}$; slot projections supply keys $K_sS_{\ell,t,k}$ and values $V_sS_{\ell,t,k}$:
\begin{equation}
\xi_{\ell,t,k}=\softmax_k\!\left((Q_sh_{\ell,t})^\top K_sS_{\ell,t,k}/\sqrt p\right),
\qquad P_{\ell,t}=\sum_{k=1}^{K}\xi_{\ell,t,k}V_sS_{\ell,t,k}.
\label{eq:controller-readout}
\end{equation}
Here $\xi$ is normalized over the $K$ slots. $S$ is the recurrent state; $P$ is its current control signal.

\Needspace{4\baselineskip}
The router and gate combine $P$ with $h$ to determine source weights (Equation~\ref{eq:routing}) and writeback strength (Equation~\ref{eq:writeback}). Source attention supplies the transported content $c$ from the bank. We use $K=8$, $p=256$, and $\gamma=0.9$; Appendix~\ref{app:slot-history} expands the history represented by $S$ and $P$.

The slots and bank reset at each forward pass and maintain separate depth trajectories for each token. Information from preceding tokens reaches them through the backbone's causal representations and, during cached decoding, its cache.

\subsection{Routing completed computations}
For every layer with a nonempty bank, routing combines the current state, controller context, and layer identity:
\begin{equation}
\begin{aligned}
q_{\ell,t}&=W_q[h_{\ell,t};P_{\ell,t};e_\ell],&
k_{b,t}&=W_k[m_{b,t};e_b^{\mathrm{src}}],\\
a_{\ell,t,b}&=\softmax_{b\in\mathcal J_\ell}
\left(q_{\ell,t}^{\top}k_{b,t}/\sqrt r\right),&
c_{\ell,t}&=\sum_{b\in\mathcal J_\ell}a_{\ell,t,b}W_v[m_{b,t};e_b^{\mathrm{src}}].
\end{aligned}
\label{eq:routing}
\end{equation}
Dense softmax mixes the visible sources in each forward pass. Source and layer embeddings identify origin and destination; the hidden state and controller context determine the query. Learned $C_b$, $W_v$, and $U_\ell$ determine how the selected content reaches the frozen layer. At a fixed receiving layer, all interventions lie in the at-most-$r$-dimensional column space of $U_\ell$ before dtype rounding, even after calibration and scalar gating. Equation~\ref{eq:transport} gives the aggregate transport form.

\subsection{Calibrating the intervention}
A gate on an unnormalized projection does not directly specify the size of its effect: the same scalar can multiply writeback directions with very different norms. We first map the retrieved source mixture to $w_{\ell,t}=U_\ell c_{\ell,t}\in\R^d$ and calibrate its RMS to the receiving state. Define $\rms(x)=\sqrt{d^{-1}\sum_jx_j^2}$ and stop-gradient $\sg$. The implemented update is
\begin{equation}
\begin{aligned}
\widehat w_{\ell,t}
&=\begin{cases}
w_{\ell,t}\dfrac{\sg(\rms(h_{\ell,t}))}{\rms(w_{\ell,t})},
&\rms(w_{\ell,t})>\epsilon,\\[-1pt]
w_{\ell,t},&\text{otherwise},
\end{cases}\\
g_{\ell,t}&=g_{\max}\tanh\!\left(b_\ell+w_g^\top[h_{\ell,t};P_{\ell,t}]\right),
\qquad R_{\ell,t}=g_{\ell,t}\widehat w_{\ell,t}.
\end{aligned}
\label{eq:writeback}
\end{equation}
We use $\epsilon=10^{-6}$ and $g_{\max}=0.05$. The signed gate can reinforce or oppose the retrieved direction. Orthogonal initialization of $U_\ell$, zero initialization of $w_g$, and $b_\ell=\operatorname{arctanh}(0.4)$ give an initial calibrated magnitude of 2\%.

When $\rms(w_{\ell,t})>\epsilon$ and $h_{\ell,t}\neq0$, calibration gives the local scale identity
\begin{equation}
\rho_{\ell,t}:=\frac{\|R_{\ell,t}\|_2}{\|h_{\ell,t}\|_2}
=|g_{\ell,t}|<g_{\max}.
\label{eq:scale-identity}
\end{equation}
The gate specifies the local relative intervention size independently of the raw projection scale. Subsequent nonlinear computation can amplify or attenuate its effect. The small-norm fallback is part of the implementation; Appendix~\ref{app:calibration} derives the active-branch Jacobian and characterizes its threshold behavior.

\subsection{Training the auxiliary module}
For mathematics, the GRPO trainer implements a single-evaluation group-relative policy gradient with a frozen-base reference. For each prompt, it samples four completions and retries a correctness-uniform group at most three additional times, retaining the first mixed-correctness group or the final attempt. Binary parsed-correctness rewards are centered and standardized within valid completions to form $A_i$; responses that are both truncated and unparseable are omitted from this preference term.

The policy is the adapted language model's completion distribution. A sequence-normalized group-relative gradient trains source transport and control end to end, with one differentiable evaluation per retained group. Appendix~\ref{app:formal} derives the objective and its credit-assignment paths. We optimize
\begin{equation}
\mathcal L=\mathcal L_{\mathrm{policy}}
+\beta\mathcal L_{\mathrm{KL}}+\lambda\Omega,
\qquad \beta=0.02,\quad\lambda=0.01.
\label{eq:objective}
\end{equation}
Here $\mathcal L_{\mathrm{KL}}$ is the sampled token-level divergence surrogate to the backbone, and $\Omega$ penalizes routing non-uniformity and source-mixture magnitude using the full-tensor reduction. Appendix~\ref{app:formal} specifies the exact surrogate, masks, and sampling distribution. Resampling changes the generated-data budget independently of the adapter. The module has 40,681,724 loss-connected parameters (0.454\% of the backbone); its allocation is 41,730,332 (0.466\%), including an unused terminal compressor and source-embedding row.

\section{Experiments}
\subsection{Evaluation setting}\label{sec:setup}
We evaluate reasoning quality, trainable allocation, and the organization of that capacity. The backbone is Qwen3.5-9B-Base \citep{qwen2026model}, with 8,953,803,264 parameters. \trm\ uses 32 layers, four-layer blocks, rank 256, and eight controller slots of dimension 256. It allocates 41.730M trainable parameters (0.466\%); 40.682M are connected to the loss.

All trained methods in the mathematics comparison use the full GSM8K training set of 7,473 prompts. Evaluation covers all 1,319 GSM8K test questions \citep{cobbe2021gsm8k}, the 500-question MATH-500 split \citep{hendrycks2021math,lightman2024verify}, and 30 questions from each of AIME 2024 and 2025. We report means $\pm$ sample standard deviations (SD) across three evaluation rounds with seeds 42, 43, and 44. Within each round, the equal-family score is
\begin{equation}
\mathrm{MathAvg}=\frac{\mathrm{GSM8K}+\mathrm{MATH\text{-}500}+\tfrac12(\mathrm{AIME24}+\mathrm{AIME25})}{3}.
\label{eq:mathavg}
\end{equation}
Aggregate SD is computed across round-level aggregates; paired differences match round indices before computing their mean and SD. Baselines include the frozen model, full-parameter GRPO, RFT, LoRA at ranks 2 and 16, LoRA-MoE with RO-GRPO, LoRA with S-GRPO, and rank-16 LoRA with matched informative retries. The last compares weight adaptation and controlled reuse under the same retry rule at a similar trainable scale. Appendix~\ref{app:implementation} details scoring and aggregation.

\subsection{Parameter-efficient mathematical reasoning}\label{sec:math-generalization}
\begin{table}[!htbp]
\caption{\textbf{Mathematical reasoning after GSM8K RL.} Accuracy (\%), mean $\pm$ SD across three evaluation seeds. AIME averages the two annual scores. LoRA-r16 (retries) matches \trm\textquotesingle s retry rule. Parameters are in millions; a dash denotes no update. Bold/underlined means are best/second-best, including ties; fewer trainable parameters are preferred.}
\label{tab:math}
\centering\small
\setlength{\tabcolsep}{4pt}
\renewcommand{\arraystretch}{1.05}
\begin{tabular}{lrrrrr}
\toprule
\rowcolor{paperhead}
\textbf{Method} & \textbf{Params (M)} & \textbf{GSM8K} & \textbf{MATH-500} & \textbf{AIME mean} & \textbf{MathAvg} \\
\midrule
Frozen base & \zerodash{0} & \stat{88.73}{0.95} & \stat{75.60}{2.25} & \stat{31.67}{3.33} & \stat{65.33}{0.89} \\
Full-parameter GRPO & 8954.000 & \stat{93.86}{0.26} & \stat{83.07}{0.70} & \stat{44.44}{5.36} & \stat{73.79}{1.83} \\
RFT & 8954.000 & \stat{91.86}{1.16} & \stat{82.93}{2.47} & \stat{46.11}{2.55} & \stat{73.64}{1.33} \\
LoRA-r2 + GRPO & \textbf{5.410} & \stat{93.51}{0.23} & \stat{84.53}{0.81} & \stat{42.22}{12.73} & \stat{73.42}{4.28} \\
LoRA-r16 + GRPO & 43.278 & \stat{93.81}{0.77} & \stat{85.13}{2.53} & \stat{45.56}{6.31} & \stat{74.83}{1.41} \\
LoRA-MoE + RO-GRPO & 173.112 & \stat{95.30}{0.61} & \stat{87.47}{1.45} & \stat{44.44}{2.55} & \stat{75.74}{1.47} \\
LoRA + S-GRPO & 43.278 & \stat{94.64}{0.46} & \stat{84.33}{1.86} & \secondstat{48.89}{8.22} & \stat{75.95}{2.43} \\
LoRA-r16 + GRPO, retries & 43.278 & \secondstat{95.58}{0.50} & \secondstat{87.93}{0.12} & \stat{48.33}{6.01} & \secondstat{77.28}{1.95} \\
\midrule\rowcolor{trmlight}
\textbf{\trm\ + GRPO} & \underline{41.730} & \beststat{97.62}{0.79} & \beststat{92.73}{0.70} & \beststat{60.56}{3.47} & \beststat{83.64}{1.16} \\
\bottomrule
\end{tabular}
\end{table}

\paragraph{Higher accuracy with 0.466\% of the backbone's parameters.}
\trm\ reaches $83.64\pm1.16$ MathAvg with 41.73M trainable parameters, versus $73.79\pm1.83$ for full-parameter GRPO over all 8.95B parameters (Table~\ref{tab:math}). The paired improvement is $9.85\pm2.93$ points. The compact interface leads all three mathematical families, combining 97.62 on GSM8K with 92.73 on MATH-500 and 60.56 mean AIME accuracy. These results link parameter efficiency to reasoning performance beyond the GSM8K training domain.

\paragraph{Stronger generalization at a comparable adaptation budget.}
Matched-retry LoRA uses 43.28M parameters and reaches $77.28\pm1.95$ MathAvg. \trm\ improves the mean by 6.36 points with 41.73M parameters. Its advantage spans MATH-500 ($+4.80$), GSM8K ($+2.04$), and AIME (about $+12.2$). The gains on broader and competition mathematics show that learning reuse supports generalization across mathematical task families. Figure~\ref{fig:reasoning-gains}a,b,d relates this improvement to trainable allocation under the matched retry rule.

\subsection{Agentic search}\label{sec:agentic-search}
\begin{table}[!htbp]
\caption{\textbf{Agentic-search comparison.} Mean $\pm$ SD across three evaluation seeds. BrowseComp Plus uses answer token-F1 $\times100$; ASearch Test uses a 0--100 score. $\Delta$ is the paired difference from full-parameter GRPO; dashes mark the reference. Rankings follow Table~\ref{tab:math}.}
\label{tab:agents}
\centering\small
\setlength{\tabcolsep}{5pt}
\renewcommand{\arraystretch}{1.05}
\begin{tabular}{lrrrr}
\toprule
\rowcolor{paperhead}
\textbf{Method} & \textbf{BrowseComp F1} & $\Delta$ vs GRPO & \textbf{ASearch} & $\Delta$ vs GRPO \\
\midrule
Frozen base & \stat{5.11}{0.30} & \stat{-28.12}{1.25} & \stat{59.76}{0.41} & \stat{-10.19}{0.56} \\
Full-parameter GRPO & \stat{33.23}{1.54} & \zerodash{0.00} & \stat{69.94}{0.53} & \zerodash{0.00} \\
RFT & \stat{31.20}{0.75} & \stat{-2.03}{2.07} & \stat{65.81}{0.92} & \stat{-4.14}{1.18} \\
LoRA-r2 + GRPO & \stat{32.98}{0.56} & \stat{-0.25}{1.25} & \stat{69.19}{0.59} & \stat{-0.75}{1.01} \\
LoRA-r16 + GRPO & \stat{36.06}{0.56} & \stat{+2.83}{1.80} & \stat{70.59}{0.18} & \stat{+0.64}{0.64} \\
LoRA-MoE + RO-GRPO & \beststat{37.00}{1.17} & \beststat{+3.77}{2.56} & \stat{71.46}{0.55} & \stat{+1.52}{0.65} \\
LoRA + S-GRPO & \stat{35.50}{0.57} & \stat{+2.27}{1.58} & \stat{70.67}{0.19} & \stat{+0.73}{0.34} \\
LoRA-r16 + GRPO, retries & \stat{36.90}{1.13} & \stat{+3.67}{1.32} & \secondstat{72.80}{0.75} & \secondstat{+2.86}{0.63} \\
\midrule\rowcolor{trmlight}
\textbf{\trm\ + GRPO} & \secondstat{36.98}{1.16} & \secondstat{+3.75}{0.91} & \beststat{73.98}{0.74} & \beststat{+4.04}{1.27} \\
\bottomrule
\end{tabular}
\end{table}

The search models are trained separately using ASearcher data \citep{asearcherdata} and the \texttt{search-agent-rl} implementation \citep{searchagentrl}, which generates multi-turn trajectories with a retrieval-and-summary tool. \trm\ reaches $36.98\pm1.16$ on BrowseComp Plus \citep{chen2025browsecompplus} and $73.98\pm0.74$ on ASearch Test (Table~\ref{tab:agents}). Its paired gains over full-parameter GRPO are $3.75\pm0.91$ F1 points and $4.04\pm1.27$ ASearch points. It achieves the highest ASearch mean, 1.18 points above matched-retry LoRA. On BrowseComp Plus, its 36.98 F1 is comparable to LoRA-MoE's 37.00 with 24.1\% of the latter's trainable allocation. These results extend the quality--allocation tradeoff to tool-mediated information search.

\begin{figure}[!htbp]
\centering
\includegraphics[width=\linewidth]{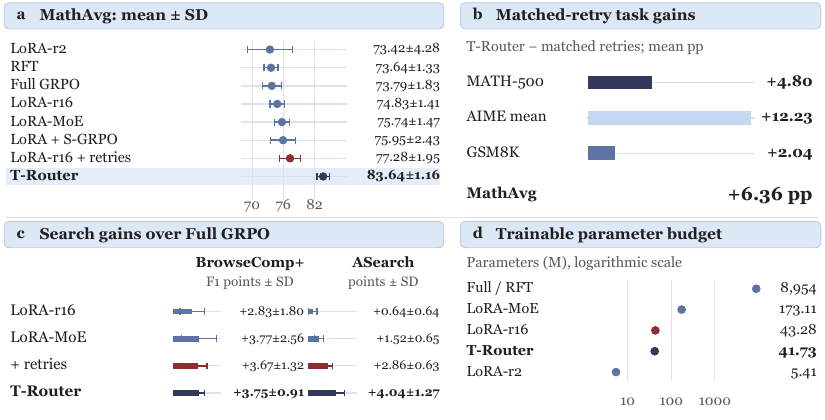}
\caption{\textbf{Reasoning quality and parameter efficiency.} (a) MathAvg mean $\pm$ SD. (b) Task-family gains over matched-retry LoRA, computed from displayed means. (c) Paired gains over full-parameter GRPO in each search metric, mean $\pm$ SD. (d) Trainable allocations on a logarithmic scale; LoRA-r16 variants share the same budget. Error bars summarize three evaluation rounds.}
\label{fig:reasoning-gains}
\end{figure}

\subsection{Addressable content and recurrent control}\label{sec:design-implications}
Table~\ref{tab:components} tests the source bank, controller, writeback, and training recipe. Figure~\ref{fig:components} aligns task responses with capacity and cost; Appendix~\ref{app:results-atlas} gives complete numerical profiles.

\begin{table}[!t]
\caption{\textbf{Components of the reuse interface.} Accuracy (\%), mean $\pm$ SD across evaluation seeds. $\Delta$ is the paired MathAvg difference from full \trm; a dash marks the reference. The MLP retains comparable parameter capacity; hidden-state memory keeps the full allocation. Appendix~\ref{app:component-results} includes both annual AIME scores.}
\label{tab:components}
\centering\small
\setlength{\tabcolsep}{6pt}
\renewcommand{\arraystretch}{1.05}
\begin{tabular}{lrrrr}
\toprule
\rowcolor{paperhead}
\textbf{Configuration} & \textbf{GSM8K} & \textbf{MATH-500} & \textbf{MathAvg} & $\Delta$ \\
\midrule
\rowcolor{trmlight}
\textbf{Full \trm} & \beststat{97.62}{0.79} & \beststat{92.73}{0.70} & \beststat{83.64}{1.16} & \zerodash{0.00} \\
Without cross-layer retrieval & \stat{96.26}{0.59} & \stat{79.40}{1.56} & \stat{69.48}{2.36} & \stat{-14.16}{1.78} \\
Hidden-state memory & \secondstat{96.79}{0.53} & \stat{77.80}{1.31} & \stat{67.46}{2.19} & \stat{-16.18}{2.96} \\
Learned static routing & \stat{96.01}{0.68} & \stat{77.27}{0.81} & \stat{64.98}{1.71} & \stat{-18.66}{2.78} \\
State replaced by matched MLP & \stat{95.02}{0.75} & \secondstat{83.93}{1.92} & \secondstat{72.80}{4.10} & \secondstat{-10.84}{3.09} \\
Without layer/block identity & \stat{94.29}{0.42} & \stat{66.13}{1.80} & \stat{58.10}{1.06} & \stat{-25.53}{0.56} \\
Without RMS alignment & \stat{94.57}{0.64} & \stat{67.93}{3.80} & \stat{56.39}{0.41} & \stat{-27.25}{1.24} \\
Without token-conditioned gate & \stat{96.26}{0.42} & \stat{66.73}{1.81} & \stat{63.78}{2.61} & \stat{-19.86}{1.47} \\
Without informative retries & \stat{94.39}{0.55} & \stat{74.87}{2.55} & \stat{62.53}{1.20} & \stat{-21.11}{2.26} \\
\bottomrule
\end{tabular}
\end{table}

\begin{figure}[!t]
\centering
\includegraphics[width=\linewidth]{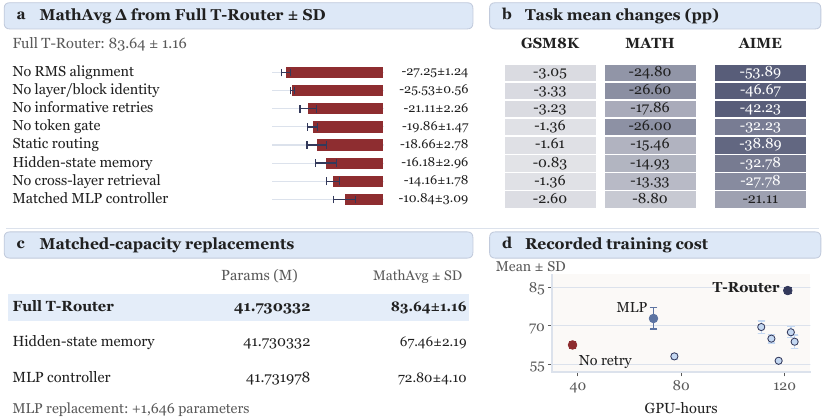}
\caption{\textbf{Dissecting controlled reuse.} (a) Paired MathAvg changes from full \trm, mean $\pm$ SD. (b) Family-wise changes in means; AIME averages both years. (c) Memory and controller replacements at matched capacity. (d) MathAvg versus recorded GPU-hours for all configurations. Numerical summaries in (c) and error bars in (d) show MathAvg SD across evaluation rounds. The no-retry setting also changes the generation budget.}
\label{fig:components}
\end{figure}

\Needspace{6\baselineskip}
\paragraph{Preserve contributions as reusable content.}
With exactly 41,730,332 parameters in both settings, block-change memory reaches $83.64\pm1.16$ MathAvg versus $67.46\pm2.19$ for hidden-state memory. Removing cross-layer retrieval yields $69.48\pm2.36$. The full interface favors a particular use of its capacity: preserve completed contributions as individually addressable sources for later computation. Compression retains each block's identity, allowing a receiver to revisit its contribution through a low-dimensional record while preserving backbone width.

\paragraph{Use depth history as routing context.}
The recurrent interface reaches $83.64\pm1.16$ MathAvg versus $72.80\pm4.10$ for a capacity-matched MLP with 41,731,978 parameters. Here $S$ accumulates depth history, and its readout $P$ conditions the router and gate while content comes from the bank.

\subsection{Source identity and relative-scale writeback}\label{sec:writeback-control}
Table~\ref{tab:components} tests the coordinates used to learn reuse. Layer and block identities specify the contribution's origin and receiving location; the full setting reaches 83.64 MathAvg versus 58.10 without them. The token-conditioned gate sets the mixture's signed influence; its removal gives 63.78. RMS alignment makes this influence relative to the receiving residual (Equation~\ref{eq:scale-identity}); omitting it gives 56.39. This calibration changes both forward intervention scale and gradient geometry, as derived in Appendix~\ref{app:calibration}. It provides a scale-defined optimization interface for learning writeback. Appendix~\ref{app:results-atlas} further characterizes the interaction between writeback and routing structure.

\subsection{Learning and executing the reuse policy}\label{sec:writeback}
\paragraph{Informative groups support learning the reuse policy.}
The full recipe records 121.23 GPU-hours and $83.64\pm1.16$ MathAvg; without informative retries, the corresponding values are 37.92 GPU-hours and $62.53\pm1.20$. Retries invest additional generation in finding mixed-reward groups that carry a policy-gradient signal. Matched-retry LoRA applies the same rule and reaches $77.28\pm1.95$ (Table~\ref{tab:math}). The capacity-matched MLP records 69.25 GPU-hours and $72.80\pm4.10$. Figure~\ref{fig:components}c,d shows these quality--cost profiles; parameter allocation and recorded training effort are distinct efficiency dimensions.

\paragraph{Summary: learning what to reuse and how strongly.}
At each receiver, source transports share $U_\ell$, constraining the intervention to an at-most-$r$-dimensional subspace (Appendix~\ref{app:formal-transport}). Source attention, recurrent context, and signed gating choose an input-dependent intervention within that space. The method comparison establishes gains across three mathematical families over matched-retry LoRA; the capacity-controlled comparisons favor block changes and recurrent context. Correctness rewards optimize these content and control pathways together. Appendix~\ref{app:results-atlas} contains complete configuration profiles, costs, auxiliary metrics, and analytical score reaggregations.

\section{Conclusion}
\trm\ improves the parameter efficiency of reasoning RL by learning how completed computations are reused. Its compressed source bank, depth-recurrent controller, and relative-scale writeback form a compact thalamic routing interface around a frozen backbone. With 0.466\% of the backbone's parameter count, it reaches 83.64 MathAvg versus 73.79 for full-parameter GRPO and 77.28 for matched-retry LoRA. Capacity-controlled comparisons favor addressable changes and recurrent context; separate search training extends the interface to tool use. The central result is an efficient allocation of adaptation capacity: preserve pretrained transformations, and learn how their contributions serve subsequent reasoning.

\subsection*{Reproducibility statement}
Section~\ref{sec:method} defines the module and training objective. Appendix~\ref{app:implementation} gives tensor shapes, exact parameter allocation, execution order, and evaluation procedures. Appendix~\ref{app:formal} supplies derivations. Appendix~\ref{app:results-atlas} includes the complete component results and costs, evaluation means and variances, configuration comparisons, auxiliary metrics, and definitions of every analytical reaggregation. The accompanying implementation includes the module, configuration constructor, training launcher, and evaluators; figure sources expose the tabulated numerical inputs.

\subsection*{AI use statement}
We used AI assistants in two roles. First, for manuscript preparation, including conceptual framing, drafting, language editing, literature retrieval, mathematical derivations, and result interpretation. Second, for coding, LaTeX production, and figure preparation, including the AI-generated brain illustration. The authors made the final decisions on the method, experimental design, and implementation, verified all AI-assisted output, and take full responsibility for this paper.

\FloatBarrier
\phantomsection
\pdfbookmark[0]{References}{references}

\clearpage
\appendix
\ThuAppendixLayout
\TRMAppendixFront

\section*{Guide to the appendices}
The appendices provide the full computational specification, mathematical analysis, and configuration-level results supporting the communication interface. Table~\ref{tab:appendix-guide} organizes the material by the question it answers.

\begin{table}[!htbp]
\centering
\caption{\textbf{Appendix guide.} Each part develops a distinct aspect of the method and study.}
\label{tab:appendix-guide}
\begin{tabularx}{\linewidth}{>{\raggedright\arraybackslash}p{4.1cm}Y r}
\toprule
\rowcolor{paperhead}
Question & Material & Page \\
\midrule
How is the module executed? & Tensor shapes, exact parameter counts, layer schedule, state lifetimes, training settings, prompts, scoring, and analytical storage & \pageref{app:implementation} \\
What properties follow from the design? & Causal source availability, transport subspaces, slot history, calibration derivatives, gradient flow, and finite group sampling & \pageref{app:formal} \\
What do the configurations reveal? & Full-test comparisons, component results and costs, writeback and initialization comparisons, routing interactions, auxiliary diagnostics, evaluation variance, and metric sensitivity & \pageref{app:results-atlas} \\
How do the components fit together? & Thalamic functional correspondence, bank/controller roles, intervention coordinates, and adaptation interfaces & \pageref{app:design} \\
\bottomrule
\end{tabularx}
\end{table}

Tables~\ref{tab:math}--\ref{tab:components} present the primary method and component comparisons. The results atlas includes their complete mean/SD profiles, sample variances, and resource measurements, alongside the historical configuration studies. Arithmetic decompositions and analytical illustrations are identified by their formulas and captions. The execution specification and derivations use zero-based decoder-layer indices, matching the source-availability schedule. Symbols retain the definitions in Section~\ref{sec:method} throughout.

\FloatBarrier
\section{Implementation and evaluation details}\label{app:implementation}
This appendix specifies the computational objects, execution schedule, parameter allocation, and training and evaluation interfaces used by \trm. The module operates on the residual input of a decoder layer. Source recording occurs after a completed block; the controller and receiving-layer intervention execute before the next frozen transformation. Keeping these events explicit gives a direct correspondence between the mathematical definition and a causal language-model implementation.

\subsection{Notation and tensor organization}
The batch dimension is $B$, the number of positions in the current forward call is $T$, and the backbone residual width is $d$. We use $L$ decoder layers, source blocks of size $s$, memory width $r$, $K$ controller slots of width $p$, controller MLP width $c$, and source/layer embedding width $e$. Canonical values are
\[
(L,s,d,r,K,p,c,e)=(32,4,4096,256,8,256,256,32).
\]
A token position is a position processed by the current call. During prompt prefill, $T$ can be the entire prompt length. During ordinary one-token cached decoding, $T=1$. All auxiliary operations broadcast over $B$ and $T$; attention over source blocks and attention over controller slots are separate from the backbone's attention over token positions.

\begin{table}[!htbp]
\centering\small
\caption{\textbf{Runtime tensors and their axes.} $J_\ell=|\mathcal J_\ell|$ is the number of completed sources visible before layer $\ell$. No source or slot axis is a token-history axis.}
\label{tab:runtime-shapes}
\begin{tabularx}{\linewidth}{l l Y}
\toprule
\rowcolor{paperhead}
Object & Tensor shape & Role \\
\midrule
Residual $H_\ell$ & $B\times T\times d$ & Input to a receiving decoder layer \\
Block anchor $A$ & $B\times T\times d$ & Post-intervention input at the start of a block \\
Block displacement $D_b$ & $B\times T\times d$ & Completed block output minus its anchor \\
Compressed memory $M_b$ & $B\times T\times r$ & Source-specific record of that displacement \\
Visible bank & $B\times T\times J_\ell\times r$ & Completed records available to a receiver \\
Controller slots $S_\ell$ & $B\times T\times K\times p$ & Depth history used to condition retrieval \\
Controller context $P_\ell$ & $B\times T\times p$ & Slot readout for the source query and gate \\
Source weights $a_\ell$ & $B\times T\times J_\ell$ & Distribution over completed source blocks \\
Source mixture $c_\ell$ & $B\times T\times r$ & Weighted sum of projected source values \\
Raw direction $W_\ell$ & $B\times T\times d$ & Receiving-layer projection of $c_\ell$ \\
Gate $g_\ell$ & $B\times T\times1$ & Signed coefficient for calibrated writeback \\
Intervention $R_\ell$ & $B\times T\times d$ & Additive input to the frozen layer \\
\bottomrule
\end{tabularx}
\end{table}

The ordering of the bank is chronological in depth. Every record keeps its own compressor and source identity even when its current routing weight is small. Layer embeddings identify receiving locations, while source embeddings identify completed blocks. Their roles are asymmetric: a receiver asks a context-dependent question about the available records, and each record provides an origin-labelled key and value. The canonical configuration retains all completed blocks within the forward call.

Controller context is the $p$-dimensional readout $P_{\ell,t}$ in Equation~\ref{eq:controller-readout}. The learned, bias-free maps $Q_s\in\R^{p\times d}$ and $K_s,V_s\in\R^{p\times p}$ are shared across layers. Attention compares $Q_sh_{\ell,t}$ with the $K$ updated slot keys, normalizes over slots, and averages their projected values. The recurrent state $S$ combines residual states, source summaries, and layer identity; the readout $P$ conditions source selection and the gate. For example, layers 4 and 5 see the same bank $\{m_0\}$, but their residuals and controller states can differ, yielding different contexts. With one source its routing weight is one, while the learned gate can still vary with $P$.

\subsection{Parameter allocation and initialization}
Table~\ref{tab:module-shapes} counts scalar parameters by component. Matrices are expressed as output-by-input dimensions. The controller MLP is
\[
\operatorname{MLP}(z)=W_2\operatorname{SiLU}(W_1z+b_1)+b_2,
\qquad z\in\R^{2p+e},
\]
with hidden dimension $c$ and output dimension $2p$. The two output halves provide a shared proposal and an elementwise proposal gate. The MLP has biases; the listed linear projections are otherwise bias free. The receiving gate has a shared vector and a separate bias for each routed layer.

\begin{table}[!htbp]
\centering\small
\caption{\textbf{Canonical module dimensions and exact parameter counts.} Counts include the allocated terminal source. The final row counts parameters with a path to the training loss.}
\label{tab:module-shapes}
\begin{tabular}{llr}
\toprule
\rowcolor{paperhead}
Component & Shape or multiplicity & Parameters \\
\midrule
Block compressors $C_b$ & $8\times(r\times d)$ & 8,388,608 \\
Writeback projections $U_\ell$ & $28\times(d\times r)$ & 29,360,128 \\
Layer and source embeddings & $32e+8e$ & 1,280 \\
Initial controller slots & $K\times p$ & 2,048 \\
Routing query $W_q$ & $r\times(d+p+e)$ & 1,122,304 \\
Routing key and value & $2\times[r\times(r+e)]$ & 147,456 \\
Hidden projection $A_h$ & $p\times d$ & 1,048,576 \\
Memory projection $A_m$ & $p\times r$ & 65,536 \\
Controller MLP & $(2p+e)\to c\to2p$ & 271,104 \\
Slot selector $W_s$ & $p\times(2p+e)$ & 139,264 \\
Slot query $Q_s$ & $p\times d$ & 1,048,576 \\
Slot key and value & $2\times(p\times p)$ & 131,072 \\
Dynamic and layer gates & $\R^{d+p}$ vector $+$ 28 scalars & 4,380 \\
\midrule
Allocated total & & 41,730,332 \\
Terminal source parameters & $rd+e$ & 1,048,608 \\
\textbf{Loss-connected total} & & \textbf{40,681,724} \\
\bottomrule
\end{tabular}
\end{table}

There are 28 receiving projections because the first four layers have no completed source to read. Eight compressors are allocated, one for each block, while the eighth source is produced after the last receiving decision. Removing that terminal compressor and its source embedding changes the allocated count by $rd+e=1,048,608$ and leaves the output unchanged. We distinguish allocation from loss connectivity so that the parameter count describes both the implemented state and the effective learning interface.

The allocated count is 0.4661\% of the 8,953,803,264-parameter backbone; the loss-connected count is 0.4544\%. Receiving projections account for 29.36M scalars and source compression accounts for 8.39M. Together they contain approximately 90.46\% of the allocation. The remainder implements routing, state update, slot reading, and scalar control. Thus most trainable capacity transforms source content into receiver-specific directions; the controller coordinates those transformations with a comparatively compact shared mechanism.

\begin{table}[!htbp]
\centering\small
\caption{\textbf{Initialization of the canonical learned-RMS module.} Initialization assigns content transport and intervention magnitude distinct roles.}
\label{tab:initialization-details}
\begin{tabularx}{\linewidth}{>{\raggedright\arraybackslash}p{3.8cm}Y}
\toprule
\rowcolor{paperhead}
Component & Initialization \\
\midrule
Source compressors, source/layer embeddings, initial slots & Gaussian with standard deviation $10^{-3}$ \\
Routing, controller-input and slot projections & Xavier initialization \\
Controller MLP output layer & Zero weights and biases \\
Receiving projections $U_\ell$ & Orthogonal initialization, gain 1 \\
Shared dynamic gate vector & Zero \\
Receiving-layer gate bias & $\operatorname{arctanh}(0.02/0.05)$ \\
Gate bound and initial coefficient & $g_{\max}=0.05$ and $g_{\mathrm{init}}=0.02$ \\
\bottomrule
\end{tabularx}
\end{table}

The zero MLP output makes the initial proposal vanish, so the slots initially follow their decay from the learned starting array. This does not zero the complete module: the source bank and orthogonal receiving projections provide a direction, and the initialized gate applies a calibrated 2\% coefficient. As optimization changes the proposal, different layers contribute different updates to the slot history. Appendix~\ref{app:slot-history} derives that accumulation.

\subsection{Layer schedule and source availability}
The distinction between a source's creation and a receiver's decision is explicit in Table~\ref{tab:depth-schedule}. Before a layer executes, its controller reads the records that already exist. The intervention is applied next. If the layer starts a block, its post-intervention residual becomes the anchor. The frozen layer then runs, and a block-ending layer creates a new source for subsequent receivers.

\begin{table}[!htbp]
\centering\small
\caption{\textbf{Source availability over 32 decoder layers.} Layer indices are zero based. Each row covers four receiver decisions; the source appended at the end is available from the next row onward.}
\label{tab:depth-schedule}
\begin{tabular}{llllr}
\toprule
\rowcolor{paperhead}
Receiving layers & Sources visible before layer & Appended after & New source & Read count \\
\midrule
0--3 & None & Layer 3 & $m_0$ & 0 \\
4--7 & $m_0$ & Layer 7 & $m_1$ & 1 \\
8--11 & $m_0,m_1$ & Layer 11 & $m_2$ & 2 \\
12--15 & $m_0,m_1,m_2$ & Layer 15 & $m_3$ & 3 \\
16--19 & $m_0,\ldots,m_3$ & Layer 19 & $m_4$ & 4 \\
20--23 & $m_0,\ldots,m_4$ & Layer 23 & $m_5$ & 5 \\
24--27 & $m_0,\ldots,m_5$ & Layer 27 & $m_6$ & 6 \\
28--31 & $m_0,\ldots,m_6$ & Layer 31 & $m_7$ & 7 \\
\bottomrule
\end{tabular}
\end{table}

The total number of source--receiver pairs per position is $4(1+2+\cdots+7)=112$. These are attention candidates, not 112 backbone layer executions: the backbone always runs exactly 32 layers. Source selection has one candidate in layers 4--7, so its normalized weight is identically one there. The receiving projection and gate still operate in those layers. Source preference becomes a nontrivial distribution at layer 8, when the second completed block is available.

At layer 4, for example, the controller has already progressed through four earlier updates. It incorporates $m_0$, reads its slots, and constructs the direction for layer 4 from the sole source. At layer 8, the same sequence of operations includes $m_0$ and $m_1$. Their attention weights can differ with the receiving token, hidden state, and accumulated controller context. This schedule separates the growth of the source set from the evolution of the state used to query it.

\begin{algorithm}[!htbp]
\caption{One causal \trm\ forward pass}\label{alg:trm}
\begin{algorithmic}[1]
\Require Residual tensor $H_0\in\R^{B\times T\times d}$, frozen layers $\mathcal B_0,\ldots,\mathcal B_{L-1}$
\State Copy the learned initial slots to all current batch/token positions
\State Set source bank $\mathcal M\gets\varnothing$ and block anchor $A\gets\varnothing$
\For{$\ell=0,\ldots,L-1$}
  \State Compute the mean visible memory; use zero when the bank is empty
  \State Update slots, then read $P_\ell$ using Equations~\ref{eq:controller-input}--\ref{eq:controller-readout}
  \If{$\mathcal M\neq\varnothing$}
    \State Compute source weights and retrieved source mixture using Equation~\ref{eq:routing}
    \State Project, calibrate and gate the writeback using Equation~\ref{eq:writeback}
  \Else
    \State $R_\ell\gets0$
  \EndIf
  \State $\widetilde H_\ell\gets H_\ell+R_\ell$
  \If{$\ell\bmod s=0$}
    \State $A\gets\widetilde H_\ell$
  \EndIf
  \State $H_{\ell+1}\gets\mathcal B_\ell(\widetilde H_\ell)$
  \If{$(\ell+1)\bmod s=0$}
    \State $b\gets\lfloor\ell/s\rfloor$
    \State Append $(C_b(H_{\ell+1}-A),e_b^{\rm src})$ to $\mathcal M$
    \State Clear the completed block anchor
  \EndIf
\EndFor
\State \Return $H_L$
\end{algorithmic}
\end{algorithm}

The post-intervention anchor matters for the meaning of a record. The writeback immediately before the block's first layer is already included in the anchor; subsequent writebacks are included in the completed displacement. A record therefore describes the change produced along the block's actual executed path. The decomposition in Appendix~\ref{app:formal} makes this distinction algebraic. It is useful when comparing a stored displacement with a hypothetical sum of frozen-layer outputs evaluated on another trajectory.

\subsection{State lifetime, precision, and differentiation}
Every forward call starts a fresh bank and copies the learned initial slot array. Runtime state also resets when the first target layer is entered. During teacher forcing, each token position has its own depth-local bank and controller. During cached generation, each new call constructs these objects for the current positions, while the backbone retains its normal causal token cache. Depth recurrence therefore operates inside a forward call; the backbone cache carries context across calls.

\begin{table}[!htbp]
\centering\small
\caption{\textbf{Objects with different lifetimes.} The same word ``memory'' can refer to trainable state, temporary depth context, or the language model's causal cache.}
\label{tab:state-lifetimes}
\begin{tabularx}{\linewidth}{lYY}
\toprule
\rowcolor{paperhead}
Object & Lifetime & Update rule \\
\midrule
Adapter parameters & Shared across examples and forward calls & Optimizer update during training \\
Initial slot array & Learned parameter & Copied at forward initialization \\
Controller state & Current forward call and token position & Recurrent update at each layer \\
Source bank & Current forward call and token position & Append after a completed block \\
Block anchor & Current block & Replace at block start; clear at block end \\
Backbone causal cache & Generation context & Managed by the backbone's decoder \\
\bottomrule
\end{tabularx}
\end{table}

The adapter computes its projections, slots, normalization, and gates in fp32. Frozen backbone tensors use bfloat16. The final intervention is cast back to the receiving residual's dtype before addition, keeping the next decoder layer's expected dtype. The target hidden-state RMS is detached; gradients remain active through the direction normalization, source weights, source content, controller, and gate inputs. This is the implementation of the geometric distinction analyzed in Appendix~\ref{app:calibration}.

Frozen backbone parameters have no parameter gradients, while their input derivatives connect earlier interventions to the final loss. Compressed records remain in the graph as well. An earlier intervention can therefore influence later losses through both the ordinary residual path and the record subsequently read by another layer. The first routed layer is layer 4; the prefix before that point supplies inputs to trainable auxiliary operations without requiring gradients for its frozen weights.

Activation checkpointing encloses the frozen decoder-layer computation. Source recording and controller updates occur outside the recomputed function. Consequently, recomputation reevaluates a frozen transformation at its saved input and does not append another source or advance the controller a second time. This placement preserves the one-update-per-layer semantics of Algorithm~\ref{alg:trm}. The adapter remains differentiable while the frozen backbone is kept in evaluation mode.

\subsection{Training configuration and optimization flow}
The final mathematics comparison trains each adaptation method on all 7,473 GSM8K training examples. Evaluation uses the full GSM8K test set, all 500 MATH-500 questions, and both 30-question AIME sets. The component study reports an estimated 3,737 optimizer steps per configuration, with exact trainable allocations and training costs in Table~\ref{tab:ablation-resources}. Prompt count, optimizer steps, and generated responses describe different aspects of this training budget.

\begin{table}[!htbp]
\centering\small
\caption{\textbf{Final mathematics comparison: training and evaluation scope.} The full-test suite is shared by the methods in Table~\ref{tab:math}.}
\label{tab:training-settings}
\begin{tabularx}{\linewidth}{lY}
\toprule
\rowcolor{paperhead}
Setting & Value \\
\midrule
Training data & GSM8K train: 7,473 questions \\
Evaluation data & GSM8K: 1,319; MATH-500: 500; AIME24/25: 30 each \\
Evaluation generation & One generated answer per question in each evaluation round \\
Evaluation aggregate & Equal weight for GSM8K, MATH-500, and mean AIME \\
T-Router allocation & 41,730,332 parameters; 40,681,724 loss connected \\
Reference policy & Frozen backbone with the adapter disabled \\
\bottomrule
\end{tabularx}
\end{table}

\begin{table}[!htbp]
\centering\small
\caption{\textbf{Optimization controls exposed by the implementation.} Values are the implementation's configurable mathematics recipe. Sampling and regularization are specified separately from the benchmark aggregation rule.}
\label{tab:implementation-controls}
\begin{tabularx}{\linewidth}{lYlY}
\toprule
\rowcolor{paperhead}
Control & Value & Control & Value \\
\midrule
Optimizer & AdamW & Weight decay & 0.01 \\
Peak learning rate & $10^{-4}$ & Warmup & 10\% of scheduled updates \\
Schedule & Linear warmup/decay & Accumulation & 2 prompt groups \\
Prompt batch size & 1 & Group size & 4 completions \\
Additional retry limit & 3 groups & Temperature & 0.7 \\
Top-$p$ & 0.95 & Completion ceiling & 512 tokens \\
KL coefficient & 0.02 & Allocation coefficient & 0.01 \\
Backbone dtype & bfloat16 & Adapter arithmetic & fp32 \\
Frozen-layer checkpointing & Enabled & Slot-diversity weight & 0 \\
\bottomrule
\end{tabularx}
\end{table}

For each prompt, generation uses the current adapted model. A group whose correctness rewards are all equal can be replaced by another group for that same prompt, with at most three replacements. The retained group is the first mixed-correctness group or the final attempt. The procedure then evaluates adapted and adapter-disabled log probabilities on the retained completions, computes group-standardized advantages, and differentiates the policy, KL, and allocation terms. The configurable recipe accumulates two prompt groups before an optimizer update. For $N_{\rm opt}$ scheduled updates, warmup lasts $\max(1,\lfloor N_{\rm opt}/10\rfloor)$ updates.

The generation distribution uses temperature and nucleus sampling, whereas the objective uses full-vocabulary model log probabilities. An explicit top-$k$ override is not applied, so generation retains the loaded model's top-$k$ setting. These are distinct parts of the procedure: sampling determines which responses enter the group, and the differentiable objective assigns gradients to those responses. The exact group selection law and finite response budget are derived in Appendix~\ref{app:formal}.

The adapter-disabled reference retains the same frozen backbone parameters. No separately trained reward model is used; the reward is defined by final-answer correctness. The configurable recipe disables parse-failure reward shaping and the optional length-format and correct-response anchoring losses. A completion that is both truncated and unparseable has zero preference weight. Other incorrect but valid completions remain in the group with reward zero. The KL term applies to the nonpadding completion tokens of the retained responses. The allocation term uses source weights $a_\ell$ and source mixtures $c_\ell$ across the current sequence, including prompt positions. The corresponding reductions are specified in Equation~\ref{eq:routing-regularizer} and the training analysis.

\subsection{Evaluation tasks, repeated rounds, and scoring}
Generative evaluations use three rounds with random seeds 42, 43, and 44. For a score $x_j$ in round $j$, the reported statistics are
\begin{equation}
\bar{x}=\frac{1}{3}\sum_{j=1}^{3}x_j,\qquad
s^2=\frac{1}{2}\sum_{j=1}^{3}(x_j-\bar{x})^2,\qquad s=\sqrt{s^2}.
\label{eq:evaluation-statistics}
\end{equation}
Tables report $\bar{x}\pm s$; $s$ is the sample SD across evaluation rounds. Within each round, AIME mean averages the two annual accuracies and MathAvg averages the three task families. For method comparisons, $d_j=x_j-y_j$ pairs matching round indices, and the same formulas give $\bar d$ and its SD. Means and SD are displayed to two decimals, after aggregation; Appendix~\ref{app:evaluation-variance} reports the complete variances to four decimals. Accuracy SD is in percentage points and its variance in squared percentage points; search scores use their corresponding points and squared points. Parameter allocations and GPU-hours are resource records, separate from these evaluation statistics.

The mathematics suite measures grade-school arithmetic, broader mathematical problem solving, and competition mathematics. Table~\ref{tab:eval-procedures} specifies the full-test denominators. Auxiliary tasks examine narrative reasoning with MuSR \citep{sprague2024musr}, domain-question answering with GPQA \citep{rein2023gpqa}, instruction compliance with IFEval \citep{zhou2023ifeval}, and language-model likelihood with WikiText-2 \citep{merity2017wikitext}. Their scores remain in their native units rather than entering MathAvg. MATH-500 uses the held-out subset released with PRM800K \citep{lightman2024verify}, accessed through the \href{https://huggingface.co/datasets/HuggingFaceH4/MATH-500}{Hugging Face interface} in Table~\ref{tab:dataset-versions}.

\begin{table}[!htbp]
\centering\small
\caption{\textbf{Full-test mathematics evaluation.} Each evaluation round covers the listed questions. AIME's two annual scores are averaged before the three-family MathAvg is computed.}
\label{tab:eval-procedures}
\begin{tabularx}{\linewidth}{l r Y}
\toprule
\rowcolor{paperhead}
Task & Questions & Scoring object \\
\midrule
GSM8K & 1,319 & Normalized final answer \\
MATH-500 & 500 & Mathematical equivalence of the final answer \\
AIME 2024 & 30 & Mathematical equivalence of the final answer \\
AIME 2025 & 30 & Mathematical equivalence of the final answer \\
\bottomrule
\end{tabularx}
\end{table}

The multiple-choice evaluator scores option likelihoods. MuSR provides narrative text, a question, and its explicit option list; the GPQA-D input contains the question and options. The evaluator tokenizes each space-prefixed option letter, requires it to be one token, and selects the highest next-token log probability. This decision rule makes the predicted option independent of the length of a generated explanation.

\begin{table}[!htbp]
\centering\small
\caption{\textbf{Dataset interfaces in the evaluation implementation.} These identifiers specify the task content accessed by the corresponding loaders.}
\label{tab:dataset-versions}
\begin{tabularx}{\linewidth}{lY}
\toprule
\rowcolor{paperhead}
Task & Dataset interface \\
\midrule
GSM8K & \texttt{openai/gsm8k} \\
MATH-500 & \texttt{HuggingFaceH4/MATH-500} \\
AIME 2024 & \texttt{HuggingFaceH4/aime\_2024} \\
AIME 2025 & \texttt{yentinglin/aime\_2025} \\
MuSR & \texttt{TAUR-Lab/MuSR} \\
GPQA-D & \texttt{fingertap/GPQA-Diamond} \\
IFEval & \texttt{google/IFEval} \\
WikiText-2 & \texttt{Salesforce/wikitext} \\
\bottomrule
\end{tabularx}
\end{table}

\subsection{Agentic search data and evaluation}\label{app:search-protocol}
Search training uses ASearcher data \citep{asearcherdata} and the \texttt{search-agent-rl} implementation \citep{searchagentrl}, which combines \texttt{verl} training, SGLang trajectory generation, and a retrieval-and-summary tool. The specifications below describe the public pipeline's data interfaces and default execution protocol. Dataset sizes count converted records before long-prompt filtering; retrieval and serving settings are repository defaults.

\paragraph{Data partitions and evaluation roles.}
The conversion reads \nolinkurl{aidenjhwu/ASearcher_en_no-math_Qwen3-8B-reject-sample}. Its 13,985-record snapshot is randomly split with seed 42 and a 2\% holdout, producing 13,706 training records and 279 held-out records. The latter file is named \texttt{test.parquet}, but the training entry point passes it to \texttt{data.val\_files} and validates every 50 optimizer steps. Accordingly, the column headed \emph{ASearch Test} in Tables~\ref{tab:agents} and~\ref{tab:search-complete} refers to \emph{ASearcher held-out validation}, rather than an additional independent test partition.

BrowseComp Plus \citep{chen2025browsecompplus} uses the 830 original questions in \texttt{Tevatron/browsecomp-plus}; the converter neither subsamples nor randomly repartitions them. Its evaluation entry point passes the same file to both \texttt{train\_files} and \texttt{val\_files}, with \texttt{val\_before\_train=True} and \texttt{val\_only=True}. The corresponding control flow returns after validation, without training updates. The retrieval collection, \texttt{Tevatron/browsecomp-plus-corpus}, contains 100,195 documents in the public snapshot. These are source sizes; prompt filtering determines the effective evaluation count, and the local index determines the available retrieval collection.

\paragraph{Exact-match split diagnostic.}
The ASearcher converter splits records without first deduplicating question text. Reconstructing the default split from the fixed public snapshot identifies eight held-out records whose question text and reference answer exactly match a training record: $8/279=2.87\%$ of the holdout. The records have different IDs, so ID-disjointness alone does not reveal these matches. This diagnostic concerns exact question--answer matches in the reconstructed, pre-filter split; it does not quantify overlap in a run-specific filtered file or semantic near-duplicates. The overlap count is a property of this validation partition and does not, by itself, attribute a difference between methods to duplication.

\paragraph{Answer extraction and token-F1.}
Both conversion entry points select the \texttt{token\_f1} scorer. It extracts the last complete \texttt{<answer>...</answer>} span and returns zero when no complete span exists. With multiple reference answers, it takes the highest answer-level F1. Whitespace is first collapsed; strings containing a space are split on spaces, and strings without spaces are split into characters, including a single English word. Case and punctuation are preserved. For token-multiset overlap $o$ and answer lengths $n_{\rm pred},n_{\rm ref}$, the nonempty-answer form is
\begin{equation}
\mathrm{Prec}=\frac{o}{n_{\rm pred}},\qquad \mathrm{Rec}=\frac{o}{n_{\rm ref}},\qquad
\operatorname{F1}=\frac{2\,\mathrm{Prec}\,\mathrm{Rec}}{\mathrm{Prec}+\mathrm{Rec}},
\end{equation}
with zero F1 when the overlap is zero. Scoring is string-based and does not use a language-model judge; the original BrowseComp-Plus benchmark's judge-based accuracy is a distinct metric.

\paragraph{Default returned score.}
The scorer's default reward applies repetition, answer-length, and answer-tag penalties to token-F1. Thus a table obtained directly from that reward has the interpretation \emph{penalized token-F1 $\times100$}:
\begin{equation*}
\operatorname{Score}=\frac{100}{N}\sum_{i=1}^{N}
  \operatorname{F1}_i\,\rho_{{\rm rep},i}\,\rho_{{\rm len},i}\,\rho_{{\rm tag},i}.
\end{equation*}
Here $\rho_{\rm rep}$ is the scorer's repetition penalty. The length factor multiplies the score by 0.85 above 3,000 scoring tokens, and by a further 0.7 above 6,000. The tag factor is 0.25 if either the opening- or closing-answer-tag count exceeds ten, and one otherwise. Scoring tokens follow the string rule above and are distinct from model tokens. A 0--100 score is a scaled mean reward, not the percentage of questions answered correctly. An unpenalized token-F1 value instead requires aggregation of the F1 component before these factors; the default returned reward must not be interpreted as that component alone.

\paragraph{Retrieval and summarization.}
Each \texttt{local\_search(query)} call retrieves the top ten documents and uses a separate model to summarize them. Table~\ref{tab:search-tool-settings} records the default settings. Both summarizers generate at most 256 model tokens with temperature 0.2 and top-$p$ 0.9. BrowseComp Plus shortens the document and retries on summary failure, with an excerpt fallback. Its launcher uses the same \texttt{BASE\_MODEL\_PATH} variable for the actor base and summary service; overriding this variable can therefore change both roles. The listed models describe the repository defaults, not an invariant of the launcher under overrides.

\begin{table}[!htbp]
\centering\small
\caption{\textbf{Default search-tool settings.} Character limits apply to the initial document excerpts; summary limits use model tokens.}
\label{tab:search-tool-settings}
\begin{tabularx}{\linewidth}{lYY}
\toprule\rowcolor{paperhead}
Setting & ASearcher training/validation & BrowseComp Plus evaluation \\
\midrule
Retrieval collection & Local \texttt{wiki-18.jsonl} & BrowseComp-Plus corpus \\
Retriever & \texttt{e5-base-v2} & \texttt{Qwen3-Embedding-8B} \\
Summary model & \texttt{Qwen3-1.7B} & \texttt{Qwen3-8B}; three services \\
Retrieved documents & Top 10 per call & Top 10 per call \\
Initial excerpt & 2,000 characters/document & 10,000 characters/document \\
Summary ceiling & 256 model tokens & 256 model tokens \\
Tool-return truncation & 2,048 characters & 2,048 characters \\
\bottomrule
\end{tabularx}
\end{table}

\paragraph{Trajectory budgets and evaluation sampling.}
The default entry points cap the prompt at 4,096 model tokens, the response or multi-turn trajectory at 35,000 token positions, and total model length at 40,000. The trajectory limit includes tool feedback. At most 100 assistant-generation rounds are permitted, with at most one tool call per round; the termination order allows at most 99 executed tool-call rounds, and length limits can stop a trajectory earlier. Tool outputs are truncated by string slicing at 2,048 characters. Training samples eight trajectories per question. Validation uses one trajectory per question, with temperature 0.7, top-$p$ 0.8, and top-$k$ 20; the training setting \texttt{rollout.n=8} does not imply best-of-eight evaluation. These are execution ceilings and sampling settings, rather than measured equal token or tool-use costs across methods.

\paragraph{Reading the aggregate comparisons.}
Tables~\ref{tab:agents} and~\ref{tab:search-complete} report three-round score means and SDs; their paired differences use corresponding evaluation-round indices. Those SDs describe variation across the reported rounds. The ASearcher column characterizes performance on the validation partition identified above, while BrowseComp Plus supplies the separate benchmark evaluation. Differences are descriptive score comparisons; evaluation-round SDs do not themselves establish item-level statistical significance.

\subsection{Prompt construction and final-answer extraction}
Prompt templates define the answer interface for each task. They are kept fixed within a comparison. GSM8K training and evaluation use a direct-answer text prompt. Mathematics tasks use a problem-solving prompt whose final answer is requested in a box. Multiple-choice tasks request the option letter directly. The task text is inserted at the indicated symbolic field below; these fields describe the template, rather than individual generated examples.

\begin{ThuPrompt}
\begin{minipage}{\linewidth}\small\ttfamily
Solve the following grade-school math problem.\\
Return only the final answer.\\[4pt]
Question: \normalfont\itshape question text\normalfont\ttfamily\\
Answer:
\end{minipage}
\end{ThuPrompt}

\begin{ThuPrompt}
\begin{minipage}{\linewidth}\small\ttfamily
Solve the following math problem carefully. End with the final answer in \textbackslash boxed\{...\}.\\[4pt]
Problem: \normalfont\itshape problem text\normalfont\ttfamily\\
Solution:
\end{minipage}
\end{ThuPrompt}

\begin{ThuPrompt}
\begin{minipage}{\linewidth}\small\ttfamily
\normalfont\itshape question and option text\normalfont\ttfamily\\[4pt]
Answer immediately with only the option letter. Do not explain or show reasoning.\\
Answer:
\end{minipage}
\end{ThuPrompt}

GSM8K reference answers are normalized from the final-answer field. The prediction parser first recognizes an explicit final-answer delimiter, then considers the region after a completed thinking segment and the output as a whole. Within a candidate region, it prioritizes an explicit final-answer line, followed by a final numeric candidate from non-step lines. Normalization removes commas, currency markers and whitespace and strips a trailing period. A mathematical answer parser selects the final expression and checks that expression against the reference answer.

For mathematical equivalence, the scorer first uses the parsed expression with the mathematical verifier when both sides parse. Its fallback normalizes basic LaTeX formatting and compares expressions symbolically; ordered tuples are compared component by component. The selected final answer is the object being judged. An intermediate occurrence of the reference expression elsewhere in the reasoning trace does not itself satisfy the final-answer criterion.

Answer generation stops at the first end-of-sequence token or at the task's completion ceiling. A response at the ceiling is marked truncated. For a truncated response that remains inside an unfinished thinking segment without an explicit final answer, the score is incorrect. An explicit final-answer region, a box, or the designated final-answer markers allow the parser to evaluate the answer that was actually emitted. This rule treats termination and answer correctness as separate observables.

\subsection{Auxiliary metric definitions}
IFEval evaluates each instruction constraint using its task-specific checker. If prompt $i$ has $n_i$ constraints with binary outcomes $z_{ij}$, its strict prompt score is
\begin{equation}
I_i=\prod_{j=1}^{n_i}z_{ij},\qquad
\operatorname{Strict@50}=\frac{100}{50}\sum_{i=1}^{50}I_i.
\end{equation}
Every prompt therefore carries equal weight, independent of how many constraints it contains. The instruction-level average is a different quantity because it weights individual constraints. The reported strict@50 metric makes satisfying the entire prompt the unit of evaluation.

WikiText-2 perplexity evaluates language-model likelihood. The evaluator joins nonempty paragraphs, tokenizes the text, and scores overlapping windows while masking context positions covered by preceding windows. Window length, stride, and token limit are evaluation configuration parameters. If a window ends at token $e_j$ after a previous end $e_{j-1}$, its new span has length $t_j=e_j-e_{j-1}$. With model-returned masked causal loss $\ell_j$, the implementation uses
\begin{equation}
\operatorname{PPL}=\exp\left(\frac{\sum_j w_j\ell_j}{\sum_jw_j}\right),
\qquad w_j=\max(t_j-1,1).
\label{eq:ppl-aggregation}
\end{equation}
The first window supplies the initial context; each later window reuses context and scores its newly exposed suffix. Perplexity is retained on its own lower-is-better scale in Table~\ref{tab:main-diagnostics}.

Final WB is a whole-tensor ratio at the latest routed layer,
\begin{equation}
\operatorname{FinalWB}=100\frac{\|g_\ell\widehat W_\ell\|_F}{\|H_\ell\|_F},\qquad \ell=31.
\label{eq:finalwb-definition}
\end{equation}
The numerator is computed before the intervention is cast to the backbone dtype. Layer-wise execution overwrites the latest writeback statistic, which makes this diagnostic a description of the final receiving location. When all positions use active calibration, $(\operatorname{FinalWB}/100)^2$ is a hidden-energy-weighted mean of squared token gates. It is consequently distinct from the arithmetic mean of gate coefficients. Appendix~\ref{app:calibration} gives the corresponding identity and its interpretation.

\subsection{Analytical storage and projection work}
The auxiliary live state has three principal components: completed compressed sources, current controller slots, and the active block anchor. Before the final receiving layers, their scalar counts per batch/token position are $7r$, $Kp$, and $d$, respectively. For the canonical dimensions this is $1792+2048+4096=7936$ fp32 scalars, or 31 KiB. This count describes those named tensors; temporary projection operands, backbone state, and differentiation history have their own allocations.

\begin{table}[!htbp]
\centering\small
\caption{\textbf{Analytical storage of source bank, current slots, and anchor.} Values use fp32 and $2^{20}$ bytes per MiB. $BT$ is the number of positions processed concurrently in one forward call.}
\label{tab:state-storage}
\begin{tabular}{rrrrr}
\toprule
\rowcolor{paperhead}
$BT$ & Source bank (MiB) & Slots (MiB) & Anchor (MiB) & Sum (MiB) \\
\midrule
1 & 0.0068 & 0.0078 & 0.0156 & 0.0303 \\
128 & 0.875 & 1.000 & 2.000 & 3.875 \\
512 & 3.500 & 4.000 & 8.000 & 15.500 \\
2048 & 14.000 & 16.000 & 32.000 & 62.000 \\
8192 & 56.000 & 64.000 & 128.000 & 248.000 \\
\bottomrule
\end{tabular}
\end{table}

Under reverse-mode differentiation, earlier controller states and auxiliary operands participate in the graph. Keeping the controller history over $L$ updates has a leading $O(BTLKp)$ scalar term. Repeated source projections additionally involve $O(BTr\sum_\ell J_\ell)$ operands. At inference, those gradient-history objects are absent, while current source and controller state remain. This distinction follows from the execution graph rather than from the allocated parameter count.

The two source-transport projection families have a particularly simple operation count. A dense $d$-to-$r$ or $r$-to-$d$ multiplication uses $dr$ scalar multiply--accumulates per position. Eight block compressions and 28 writebacks therefore contribute
\begin{equation}
8dr+28dr=36dr=37,748,736
\end{equation}
multiply--accumulates per batch/token position in the allocated execution. Omitting the terminal compression gives $35dr=36,700,160$. Routing queries, source key/value projections, controller updates, and slot attention add their respective matrix and reduction operations. These formulas expose which dimensions control the auxiliary work: $r$ scales transport width, $Kp$ scales live controller state, and $s$ sets the number and spacing of source records.

Unlike a fixed linear weight update, the intervention depends on the receiving state, visible records, controller history, and calibrated direction. Its computation remains part of the forward pass after training. The storage and operation counts above therefore characterize the resources of an explicit communication interface: a small trainable footprint with transient per-position state and context-dependent computation.

\subsection{How architectural dimensions allocate capacity}
The parameter count can be evaluated for other source ranks and block sizes without training a new model. This calculation separates architectural capacity from measured task quality. Let $M=L/s$ for a block size that divides $L$. With all layers targeted and the same dense controller structure, the allocated count decomposes as
\begin{equation}
\begin{aligned}
P_{\rm transport}&=(M+L-s)dr,\\
P_{\rm route}&=r(d+p+e)+2r(r+e),\\
P_{\rm controller}&=2pd+pr+c(2p+e)+c+2pc+2p\\
&\quad+p(2p+e)+2p^2+Kp,\\
P_{\rm identity}&=e(L+M),\qquad
P_{\rm gate}=d+p+L-s.
\end{aligned}
\label{eq:general-parameter-count}
\end{equation}
Their sum gives 41,730,332 at the canonical dimensions. The receiving and compression terms grow linearly with rank, while the source key/value term also has a quadratic component. Controller width is held fixed in Figure~\ref{fig:analytical-parameters}, which isolates the allocation changes induced by source rank and block size.

\begin{figure}[!htb]
\centering
\includegraphics[width=\linewidth]{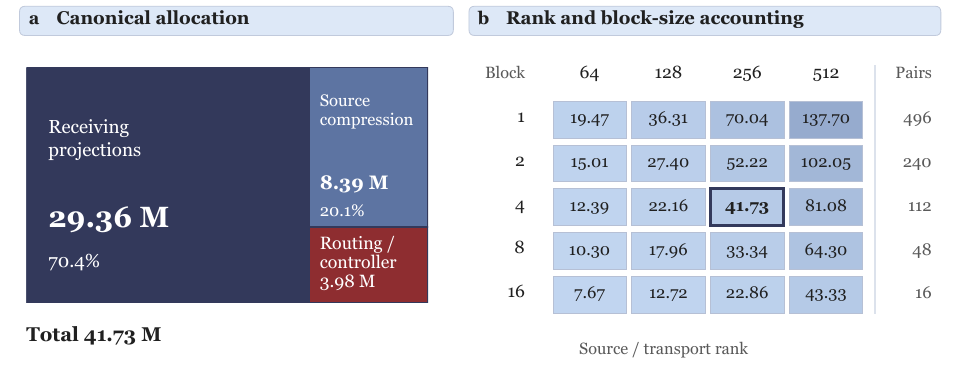}
\caption{\textbf{Analytical allocation of adaptation capacity.} Left: area-proportional component allocation in the canonical module. Right: Equation~\ref{eq:general-parameter-count} over source ranks and block sizes, with $L=32$, $d=4096$, $K=8$, $p=c=256$, and $e=32$. Cells give parameters in millions; the outlined cell is canonical. The last column counts source--receiver pairs.}
\label{fig:analytical-parameters}
\end{figure}

A smaller block size creates more sources and enables earlier writeback, increasing transport allocation and source--receiver pairs, $sM(M-1)/2$. A larger block compresses a longer interval into each record, reducing the resolution of the addressable history.

Memory rank sets the width of each stored change and the maximal receiving subspace dimension. Controller width sets the capacity used to select among records. Their equal canonical width, $r=p=256$, does not make these roles interchangeable; Figure~\ref{fig:analytical-parameters} quantifies the resulting allocation.

\clearpage
\section{Formal analysis of cross-layer communication}\label{app:formal-analysis}\label{app:formal}

This section analyzes the interface through which \trm\ reuses completed computations. We first establish its event ordering and the content of its block records. We then characterize the receiving subspace, the history represented by controller slots, the geometry of calibrated interventions, and their propagation through frozen layers. The final part derives the differentiated training objective and the distribution induced by bounded informative resampling. These results concern the specified computation in exact arithmetic; casts to the backbone dtype occur after the calibrated intervention is formed.

\subsection{Notation and completed-block event ordering}\label{app:formal-events}

We use sequence tensors $H_\ell\in\R^{B\times T\times d}$ and token vectors $h_{\ell,t}\in\R^d$, suppressing the batch index unless a reduction requires it. The norm of a sequence tensor is its Frobenius norm. Each $\mathcal B_\ell$ is a complete frozen decoder layer, including causal token mixing and its internal residual connections. Auxiliary projections and attentions are tokenwise. The $K$ controller slots for one token form the rows of $\mathbf S_{\ell,t}\in\R^{K\times p}$. Layer indices start at zero; $\mathbf S_{-1,t}=\mathbf S^0$ denotes the learned initial state.

For block size $s$, the events at depth $\ell$ occur in the order
\begin{equation}
\begin{gathered}
\text{read completed records}\;\longrightarrow\;
\text{update and read controller}\;\longrightarrow\;R_\ell,\\
\widetilde H_\ell=H_\ell+R_\ell
\;\longrightarrow\;H_{\ell+1}=\mathcal B_\ell(\widetilde H_\ell)
\;\longrightarrow\;\text{append if the block ends}.
\end{gathered}
\label{eq:formal-event-order}
\end{equation}
An anchor $\widetilde H_{sb}$ is associated with the first actual input of block $b$. Its compressed record is appended after layer $s(b+1)-1$ finishes. Thus the readable set before layer $\ell$ is
\begin{equation}
\mathcal J_\ell=\{0,\ldots,\lfloor\ell/s\rfloor-1\},\qquad
n_\ell:=|\mathcal J_\ell|=\lfloor\ell/s\rfloor,
\label{eq:formal-bank-cardinality}
\end{equation}
with the set empty when $\ell<s$. This formula assumes the full-history configuration and $s\mid L$. It distinguishes storing a record from using it: block $b$ has exactly $L-s(b+1)$ receiving layers, when this quantity is positive.

\begin{ThuStatement}
\begin{proposition}[Causal source availability]\label{prop:causality}
Assume each frozen layer is causal over token positions. Before layer $\ell$, the residual, controller slots, and visible records at position $t$ depend only on input positions at most $t$. Routing reads only completed blocks and does not depend on the output of its receiving layer.
\end{proposition}
\end{ThuStatement}
\begin{proof}
At the first layer the input embeddings are causal, the slots are copied from an input-independent parameter array, and the bank is empty. Suppose the claim holds before layer $\ell$. Every auxiliary operation at position $t$ uses $h_{\ell,t}$, slots at that position, and compressed records at that position. Linear maps, nonlinearities, slot attention, source attention, and RMS calibration therefore preserve the same token support. Applying the causal map $\mathcal B_\ell$ also preserves it. A new record, if created, is a tokenwise linear map of the difference between this output and an earlier causal anchor. Equation~\ref{eq:formal-event-order} places that append after the receiving decision, so it is available only to later layers. Induction establishes both token causality and depth causality.
\end{proof}

The argument also explains the state lifetime. A new forward pass initializes fresh slots and an empty bank. In cached decoding, earlier tokens influence the current token through the backbone's causal cache; the auxiliary state evolves along the current pass's depth. Under exact arithmetic, identical position/mask semantics, and cache-equivalent frozen layers, teacher-forced and cached evaluations construct the same token-local auxiliary state for the same prefix. This follows by applying the induction to the current token at each depth. It does not require persistent controller slots across generated tokens.

For $L/s=m$ complete blocks, the total number of source--receiver pairs is
\begin{equation}
\sum_{\ell=0}^{L-1}n_\ell
=s\sum_{j=0}^{m-1}j
=\frac{s m(m-1)}2.
\label{eq:formal-read-count}
\end{equation}
The canonical configuration has 112 pairs. The terminal record has no consumer; deleting its append leaves $H_L$ and the training loss unchanged. This is an execution identity rather than an approximation.

\subsection{What a block change contains}\label{app:formal-delta}

Write $\mathcal B_j(X)=X+F_j(X)$, where $F_j$ includes all computations within the complete layer beyond its outer identity term. This is an algebraic decomposition of the frozen layer, not a replacement for its internal architecture. The actual residual transition is
\begin{equation}
H_{j+1}-H_j=R_j+F_j(\widetilde H_j).
\label{eq:formal-layer-change}
\end{equation}
Summing across block $b$ and subtracting the post-intervention anchor gives
\begin{equation}
\begin{aligned}
D_b
&=H_{s(b+1)}-\widetilde H_{sb}\\
&=\sum_{j=sb}^{s(b+1)-1}F_j(\widetilde H_j)
  +\sum_{j=sb+1}^{s(b+1)-1}R_j.
\end{aligned}
\label{eq:delta-decomposition}
\end{equation}
Indeed, the telescoping sum first includes $R_{sb}$, and the subtraction of $\widetilde H_{sb}=H_{sb}+R_{sb}$ removes precisely that term. All later injections within the block remain. For a two-layer block, for example,
$D_b=F_{2b}(\widetilde H_{2b})+R_{2b+1}+F_{2b+1}(\widetilde H_{2b+1})$.

The first injection can still influence every $F_j$ through its input. Equation~\ref{eq:delta-decomposition} excludes its \emph{explicit additive term}, not its downstream computational effect. The record therefore describes the change realized by the intervened block. It is neither a separately evaluated unmodified-backbone change nor a causal attribution to an isolated layer. This distinction makes the record usable during one forward pass: subsequent routing operates on computations the model has actually completed.

Two immediate consequences are useful. First, compression yields
\begin{equation}
\|m_b\|_F\le\|C_b\|_2
\left(\sum_{j=sb}^{s(b+1)-1}\|F_j(\widetilde H_j)\|_F
+\sum_{j=sb+1}^{s(b+1)-1}\|R_j\|_F\right).
\label{eq:formal-memory-bound}
\end{equation}
Second, both endpoints remain differentiable. If $G_b=\partial\mathcal L/\partial m_b$ denotes the adjoint from all later uses, the direct contribution to the compressor gradient is
$\nabla_{C_b}\mathcal L=\sum_{n,t}G_{b,n,t}D_{b,n,t}^{\top}$, and the adjoint entering $D_b$ is $C_b^\top G_b$. The endpoint difference distributes that adjoint with opposite signs to output and anchor, after which ordinary reverse-mode differentiation follows their shared earlier computation. Shared ancestors are handled by summing these paths, not by treating the two endpoints as independent model executions.

\subsection{Source transport and the receiving subspace}\label{app:formal-transport}

Partition the routing value projection into content and identity terms, $W_v=[W_v^m\;W_v^e]$. For one receiving layer and token, the raw direction is
\begin{equation}
w_{\ell,t}=\sum_{b\in\mathcal J_\ell}a_{\ell,t,b}
\left[\underbrace{U_\ell W_v^mC_b}_{T_{\ell b}}D_{b,t}
+U_\ell W_v^e e_b^{\mathrm{src}}\right],
\quad\operatorname{rank}(T_{\ell b})\le r.
\label{eq:transport}
\end{equation}
Conditioning on the attention weights, this is affine in the collection of block changes. The content transports $T_{\ell b}$ can differ across sources through $C_b$, while all terms share the receiving projection $U_\ell$. In particular, the concatenated content transport $[T_{\ell0}\;T_{\ell1}\;\cdots]$ also has rank at most $r$, rather than the sum of the individual rank bounds.

\begin{ThuStatement}
\begin{proposition}[Aggregate receiving subspace]
At fixed parameters and receiving layer $\ell$, every intervention lies in $\operatorname{col}(U_\ell)$:
\begin{equation}
\operatorname{span}\{R_{\ell,t}:\text{inputs and positions }t\}
\subseteq\operatorname{col}(U_\ell),\qquad
\dim\operatorname{col}(U_\ell)\le r.
\label{eq:writeback-subspace}
\end{equation}
This holds for both branches of RMS calibration and for signed scalar gates.
\end{proposition}
\end{ThuStatement}
\begin{proof}
Source attention first produces $c_{\ell,t}\in\R^r$. Calibration and gating multiply $U_\ell c_{\ell,t}$ by a scalar $\alpha_{\ell,t}$, with $\alpha=g\rms(h)/\rms(w)$ on the active branch and $\alpha=g$ on the fallback branch. Thus $R_{\ell,t}=U_\ell(\alpha_{\ell,t}c_{\ell,t})$. Every such vector belongs to the same column space, including zero vectors and source-identity contributions. Taking a span proves the result.
\end{proof}

There is a corresponding differential statement. Let $x$ collect any continuous upstream inputs to a token's routing decision, and write $f_\ell(x)=\alpha_\ell(x)c_\ell(x)$. At any differentiable point away from a calibration boundary,
\begin{equation}
\frac{\partial R_\ell}{\partial x}
=U_\ell\frac{\partial f_\ell}{\partial x},\qquad
\operatorname{rank}\!\left(\frac{\partial R_\ell}{\partial x}\right)\le r.
\label{eq:formal-interface-jacobian}
\end{equation}
Adaptive attention changes the coefficient derivative but not its receiving column space. Holding the auxiliary history fixed and differentiating the layer-input interface gives
$\partial\widetilde h_\ell/\partial h_\ell=I+U_\ell\partial f_\ell/\partial h_\ell$.
This locates the low-dimensional modification at the interface itself. The subsequent nonlinear $\mathcal B_\ell$ can transform it, and comparing whole-layer Jacobians at two different inputs introduces the change in the backbone Jacobian as well.

\paragraph{Addressable records and coefficient choice.}
For fixed visible values $v_b=W_v[m_b;e_b^{\rm src}]$, dense softmax gives $c\in\operatorname{conv}\{v_b\}$; with finite logits all weights are positive. At the raw interface $w$ lies in the corresponding convex hull of $U_\ell v_b$. Calibration preserves the retrieved direction, while the signed gate sets the final magnitude and sign. Thus dense attention permits a state-dependent mixture without implying sparse execution or omission of a frozen layer. When there is one record, $a_0=1$ and the attention-logit derivative is zero, although the content/value path is still trainable. With additional records, the controller and current state can change the mixture within the same receiving subspace.

The interface is generally nonlinear in its input: attention weights, controller history, normalization, and gating all depend on that input. A fixed rank-$r$ transport and an input-dependent rank-$r$ interface share a dimensional constraint but need not implement the same function. For example, with one nonzero value and a scalar gate depending on $h$, $R(h)=g(h)\widehat w(h)$ already varies nonlinearly. It cannot in general be absorbed into one constant additive weight matrix.

\subsection{Controller history, rank, and state bounds}\label{app:slot-history}

Suppressing the token index, define $\nu_\ell=\sigma(v_\ell)\odot\tanh(u_\ell)\in(-1,1)^p$ and let $\omega_\ell$ be the slot-selection probability vector. We take $0\le\gamma<1$, with $\gamma=0.9$ in the canonical configuration. The row-stacked update is
\begin{equation}
\begin{aligned}
\mathbf S_\ell
&=\gamma\mathbf S_{\ell-1}+(1-\gamma)\omega_\ell\nu_\ell^\top\\
&=\gamma^{\ell+1}\mathbf S^0
+(1-\gamma)\sum_{j=0}^{\ell}\gamma^{\ell-j}\omega_j\nu_j^\top.
\end{aligned}
\label{eq:slot-history}
\end{equation}
The second equality follows by substitution, starting from $\mathbf S_{-1}=\mathbf S^0$. It is valid even though $\omega_j$ and $\nu_j$ depend on the realized preceding states. Unrolling a trajectory is an algebraic operation; it does not assume independent updates.

\paragraph{Rank and distinct history mixtures.}
Each added matrix has rank at most one, so subadditivity gives
\begin{equation}
\operatorname{rank}(\mathbf S_\ell)
\le\min\{K,p,\operatorname{rank}(\mathbf S^0)+\ell+1\}.
\label{eq:formal-slot-rank}
\end{equation}
To see why the state need not remain rank one, take $K=p=2$, $\mathbf S^0=(1,2)^\top e_1^\top$, scaled selector $W_sz/\sqrt p=e_1$, and a nonzero proposal proportional to $e_2$. The selector produces $\omega=\softmax(1,2)$, whose entries are not in the ratio $1:2$. Therefore the two columns of
$\gamma(1,2)^\top e_1^\top+(1-\gamma)\omega\nu^\top$
are linearly independent for $0<\gamma<1$. A rank-one initial state becomes rank two in one valid shared-proposal update. Both proposal coordinates and selector scale can be chosen within the specified parameterization.

There are also exact symmetry cases. If all initial rows are equal, their slot logits are equal, so $\omega_k=1/K$. The shared proposal preserves equal rows; by induction, the slots remain identical. If $\omega_j$ is a fixed vector at every depth and the initial state has the same row factor, the complete history also has that row factor and rank at most one. These cases show what produces multiple slot mixtures: distinct initial rows and depth-dependent slot allocation, together with different proposal directions. The learned initialization allows this asymmetry without assigning a prescribed semantic role to any slot.

For a fixed trajectory, substituting Equation~\ref{eq:slot-history} into the attention readout (Equation~\ref{eq:controller-readout}) expresses the controller context in terms of layer-wise proposals:
\begin{equation}
P_\ell=\gamma^{\ell+1}\sum_k\xi_{\ell,k}V_sS_k^0
+(1-\gamma)\sum_{j=0}^{\ell}\gamma^{\ell-j}
\underbrace{(\xi_\ell^\top\omega_j)}_{\text{read--write overlap}}V_s\nu_j.
\label{eq:formal-slot-read-history}
\end{equation}
The scalar overlap is between zero and one. Each past proposal contributes according to both its earlier slot assignment and the current slot read. Consequently, a decay average of the proposals alone does not generally determine $P_\ell$: the history of slot assignments also matters. At the same time, the slots are finite-dimensional mixtures of that history, not an unbounded list of separately addressable depth records.

\paragraph{Forward bounds and conditional sensitivity.}
The bounded proposal and probability selector imply
\begin{equation}
\begin{aligned}
\|\mathbf S_\ell\|_F
&\le\gamma^{\ell+1}\|\mathbf S^0\|_F
+(1-\gamma^{\ell+1})\sqrt p,\\
\sum_k|S_{\ell,kj}|
&\le\gamma^{\ell+1}\sum_k|S^0_{kj}|+1-\gamma^{\ell+1}.
\end{aligned}
\label{eq:formal-slot-forward-bound}
\end{equation}
For the first line, $\|\omega\nu^\top\|_F=\|\omega\|_2\|\nu\|_2\le\sqrt p$; apply the triangle inequality to the unrolled history. For the second, sum absolute values over slots and use $\sum_k\omega_k=1$. Since readout attention is a probability vector,
$\|P_\ell\|_2\le\|V_s\|_2\max_k\|S_{\ell,k}\|_2$.
These bounds apply to forward values for any inputs and selector logits.

Sensitivity involves the selector as well as $\gamma$. Hold $z_\ell$ fixed and write $a=W_sz_\ell/\sqrt p$, $\omega=\softmax(\mathbf S a)$. The differential of one update with respect to the previous slots is
\begin{equation}
\Delta\mathbf S^+
=\gamma\Delta\mathbf S
+(1-\gamma)(\operatorname{diag}\omega-\omega\omega^\top)
(\Delta\mathbf S a)\nu^\top.
\label{eq:formal-slot-differential}
\end{equation}
The softmax Jacobian has operator norm at most $1/2$: its quadratic form is the variance of the entries of a vector under $\omega$, at most one quarter of their squared range, and that squared range is at most twice the vector's squared norm. Hence
\begin{equation}
\|\Delta\mathbf S^+\|_F
\le\left[\gamma+\frac{1-\gamma}{2}\|a\|_2\|\nu\|_2\right]
\|\Delta\mathbf S\|_F.
\label{eq:formal-slot-local-gain}
\end{equation}
This is a contraction bound when $\|a\|_2\|\nu\|_2<2$ with fixed $z_\ell$. In the full model $z_\ell$ also depends on the current residual and visible records, so its derivative supplies additional paths. The decay factor controls historical coefficients; it is not by itself the Jacobian norm of the complete controller.

\subsection{RMS geometry, gradients, and threshold behavior}\label{app:calibration}

For this subsection suppress layer and token indices, put $q=\rms(w)$ and $\kappa=\rms(h)$, and define the calibrated map $f_\kappa(w)=\widehat w$ from Equation~\ref{eq:writeback}. On the active branch $q>\epsilon$, $\rms(x)=\|x\|_2/\sqrt d$ gives
\begin{equation}
f_\kappa(w)=\|h\|_2\frac{w}{\|w\|_2},\qquad
\|f_\kappa(w)\|_2=\|h\|_2.
\label{eq:formal-sphere-map}
\end{equation}
Calibration maps the raw direction to the sphere whose radius is the receiving state's norm. Positive rescaling of $w$ cancels while both points stay on this branch; negative rescaling reverses the direction. Multiplication by the signed gate gives $\|R\|_2/\|h\|_2=|g|$ for $h\ne0$. This explains how gate magnitude can be interpreted independently of the scale of the compression and projection matrices.

\paragraph{Tangential and radial derivatives.}
Holding the reference scale fixed, differentiation gives
\begin{equation}
J_w:=\frac{\partial f_\kappa(w)}{\partial w}
=\frac{\kappa}{q}\left(I-\frac{ww^\top}{\|w\|_2^2}\right),\quad
J_ww=0,\quad \|J_w\|_2=\frac{\kappa}{q}\quad(d>1).
\label{eq:rms-jacobian}
\end{equation}
To derive it, $\mathrm{d}q=w^\top\mathrm{d}w/(d\,q)$, where $d$ is the vector dimension. Applying the product rule to $\kappa w/q$ yields
$\mathrm{d}f=(\kappa/q)\mathrm{d}w-\kappa w(w^\top\mathrm{d}w)/(d\,q^3)$; since $d\,q^2=\|w\|_2^2$, Equation~\ref{eq:rms-jacobian} follows. The matrix in parentheses is the orthogonal projector onto the tangent space perpendicular to $w$. Its radial eigenvalue is zero and its $d-1$ tangential eigenvalues are one. Thus calibration learns direction through tangential gradients; the signed gate supplies a separate radial degree of freedom in the final intervention.

\begin{figure}[!htbp]
\centering
\includegraphics[width=\linewidth]{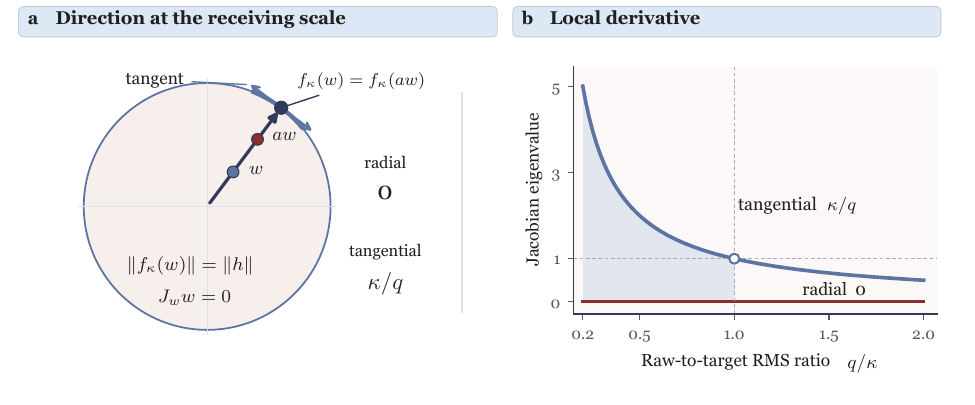}
\caption{\textbf{Analytical geometry of RMS calibration.} Left: in two dimensions with $\|h\|=1$, positive multiples on the same ray map to one point on the target sphere; the double arrow marks a tangent. The local basis separates zero radial response from tangential gain $\kappa/q$. Right: the active-branch Jacobian eigenvalues from Equation~\ref{eq:rms-jacobian}. The gate subsequently sets signed amplitude.}
\label{fig:calibration-geometry}
\end{figure}

Figure~\ref{fig:calibration-geometry} separates three geometric objects: the ray selected by the raw projection, the sphere set by the receiving residual, and the local derivative along or across that ray. A positive radial change of the raw vector leaves its calibrated output fixed. A tangential change rotates the direction, with sensitivity determined by $\kappa/q$. The gate's derivative supplies the remaining scalar control of the final intervention. This geometry connects the parameterization to the separate source, transport, and strength coordinates discussed in Appendix~\ref{app:design}.

Let $a=\partial\mathcal L/\partial R$ be an incoming adjoint and temporarily hold routing inputs fixed. Then
\begin{equation}
\frac{\partial\mathcal L}{\partial w}=gJ_w a,
\qquad
\frac{\partial\mathcal L}{\partial g}=a^\top f_\kappa(w).
\label{eq:formal-gate-direction-gradients}
\end{equation}
The normalization derivative can have gain greater than one when the raw RMS is small relative to the reference RMS. Meanwhile, for gate logit $z$, $\partial g/\partial z=g_{\max}\operatorname{sech}^2z\le g_{\max}$. These two derivatives control different quantities: the first changes the direction, while the second changes its signed coefficient.

\begin{table}[htbp]
\centering\small
\caption{\textbf{Calibration regimes for fixed $\kappa$.} Values and derivatives refer to the piecewise mathematical map before the final dtype cast.}
\label{tab:formal-rms-regimes}
\begin{tabular}{llll}
\toprule
\rowcolor{paperhead}
Region & Returned vector & Output norm & Derivative in $w$ \\
\midrule
$q>\epsilon$ & $(\kappa/q)w$ & $\sqrt d\kappa$ & Tangential projector $\times\kappa/q$ \\
$q<\epsilon$ & $w$ & $\sqrt d q$ & $I$ \\
$q=\epsilon$ & $w$ & $\sqrt d\epsilon$ & Boundary; see one-sided limits \\
\bottomrule
\end{tabular}
\end{table}

\paragraph{The threshold and a branch-aware bound.}
For a unit vector $v$, the fallback value at $w=\sqrt d\epsilon v$ is $\sqrt d\epsilon v$, while the active-side limit is $\sqrt d\kappa v$. Their difference has norm $\sqrt d|\kappa-\epsilon|$. The forward map is continuous at this boundary only if $\kappa=\epsilon$; even then the radial one-sided derivatives differ. Inside the fallback region the map is the identity. These properties follow from the explicit branch rule, which returns the unnormalized vector below the threshold rather than dividing it by $\epsilon$.

For example, an \emph{analytical illustration} with $\kappa=1$ and $\epsilon=10^{-6}$ has active-side tangential gain approaching $10^6$, while its active output RMS remains one. The two numbers describe derivative gain and forward scale respectively. For a complete sequence, let $\mathcal A$ and $\mathcal F$ be the active and fallback token positions. Since $|g|<g_{\max}$,
\begin{equation}
\|R_\ell\|_F^2
\le g_{\max}^2\left(\sum_{(n,t)\in\mathcal A}\|h_{\ell,n,t}\|_2^2
+|\mathcal F|d\epsilon^2\right).
\label{eq:formal-mixed-branch-bound}
\end{equation}
The fallback has a small absolute bound; its relative bound also depends on the receiving norm. When every position is active and $\|H_\ell\|_F>0$, the squared norm ratio satisfies
\begin{equation}
\frac{\|R_\ell\|_F^2}{\|H_\ell\|_F^2}
=\sum_{n,t}\frac{\|h_{\ell,n,t}\|_2^2}{\|H_\ell\|_F^2}
g_{\ell,n,t}^{\,2}.
\label{eq:formal-tensor-gate-rms}
\end{equation}
Thus the norm ratio is the square root of a hidden-energy-weighted mean of squared gates.

\subsection{Stop-gradient and scale-invariant optimization}\label{app:formal-gradients}

The reference RMS is detached when forming $f_\kappa$. Its forward value still tracks $h$, but its reverse-mode derivative is zero. On the active branch $q>\epsilon$, for $h\ne0$ and with $w$ held fixed, the mathematical forward map without detachment would contribute the additional Jacobian
\begin{equation}
\frac{\partial f_{\rms(h)}(w)}{\partial h}
=\frac{w}{q}\frac{h^\top}{d\kappa}.
\label{eq:formal-detached-edge}
\end{equation}
Detachment removes this reference-scale edge. It leaves the identity path through $\widetilde h=h+R$, the dependence of $w$ on $h$ and past computations, and the dependence of $g$ on $h$ and $P$ intact. On the active branch the differentiated writeback is therefore
$dR=f_\kappa(w)\,dg+gJ_w\,dw$, with no $d\kappa$ term.

Memory construction is not detached. A later loss can reach an earlier intervention through the ordinary residual chain, the compressed source record, and the controller that reads that record. Freezing a layer's parameters preserves its input vector--Jacobian product:
\begin{equation}
\frac{\partial\mathcal L}{\partial\widetilde H_\ell}
=J_{\mathcal B_\ell}(\widetilde H_\ell)^\top
\frac{\partial\mathcal L}{\partial H_{\ell+1}}
\quad\text{for the backbone edge}.
\label{eq:formal-frozen-vjp}
\end{equation}
Other outgoing edges add their adjoints to this expression. This is also why training earlier LoRA modules differentiates through subsequent frozen operators. Parameter freezing removes weight-gradient and optimizer-state requirements; it does not replace reverse-mode propagation by a local update at each insertion point.

\paragraph{Positive scaling of a receiving projection.}
Suppose all uses of $U_\ell$ remain on the active branch in a neighborhood of a positive rescaling, and the loss has no explicit norm penalty on $U_\ell$. Replacing $U_\ell$ by $aU_\ell$, $a>0$, scales the raw $w$ by $a$ and leaves every calibrated intervention unchanged. Downstream residuals, records, and losses are consequently unchanged as well. Differentiating this invariance at $a=1$ gives
\begin{equation}
\langle\nabla_{U_\ell}\mathcal L,U_\ell\rangle_F=0.
\label{eq:formal-radial-gradient}
\end{equation}
This concerns the Euclidean gradient of the differentiated loss. AdamW adds momentum, coordinatewise preconditioning, and decoupled decay. If $V$ is its preconditioned moment direction, $\eta$ its learning rate, and $a=1-\eta\lambda_{\rm wd}$, the actual norm update is
\begin{equation}
U^+=aU-\eta V,\qquad
\|U^+\|_F^2=a^2\|U\|_F^2-2a\eta\langle U,V\rangle_F+\eta^2\|V\|_F^2.
\label{eq:adamw-scale}
\end{equation}
Raw-gradient orthogonality does not force $\langle U,V\rangle_F$ to vanish. Even when it vanishes, the last term is positive. For a concrete two-coordinate illustration, $U=(1,2)$ and raw gradient $(2,-1)$ have zero inner product; a first-step Adam direction close to $(1,-1)$ has inner product $-1$ with $U$. Reversing the raw gradient reverses that inner product. The projection norm can therefore increase or decrease under the full update. Decay contracts its own component; it does not establish a monotone trajectory for the total norm.

\subsection{Conditional propagation through the frozen decoder}\label{app:formal-propagation}

Local calibration supplies an intervention budget at each receiving layer. Its effect at a later depth depends on the frozen computation between those layers. To make this dependence explicit, compare an intervened trajectory $H_\ell$ with the adapter-disabled trajectory $H^0_\ell$, starting from the same $H_0$. Assume $\mathcal B_\ell$ is $\Lambda_\ell$-Lipschitz on a domain containing the two actual layer inputs. Put $e_\ell=\|H_\ell-H^0_\ell\|_F$. Then
\begin{equation}
\begin{aligned}
e_{\ell+1}
&\le\Lambda_\ell\|H_\ell+R_\ell-H^0_\ell\|_F\\
&\le\Lambda_\ell(e_\ell+\|R_\ell\|_F),\\
e_L&\le\sum_{j=0}^{L-1}
\left(\prod_{k=j}^{L-1}\Lambda_k\right)\|R_j\|_F.
\end{aligned}
\label{eq:formal-forcing-bound}
\end{equation}
The final inequality follows by repeatedly substituting the preceding one and using $e_0=0$. Crucially, it treats each $R_j$ as the intervention realized on the actual trajectory. Routing can depend on the entire available auxiliary history; no assumption that $R_j$ is a function of $H_j$ alone is needed.

When all positions use active calibration and $|g_{j,n,t}|\le\rho_j$, we have $\|R_j\|_F\le\rho_j\|H_j\|_F$. One can either substitute this directly into Equation~\ref{eq:formal-forcing-bound}, or eliminate the intervened norm using $\|H_j\|_F\le\|H_j^0\|_F+e_j$. The latter gives
\begin{equation}
e_L\le\sum_{j=0}^{L-1}\Lambda_j\rho_j\|H_j^0\|_F
\prod_{k=j+1}^{L-1}\Lambda_k(1+\rho_k).
\label{eq:formal-closed-propagation-bound}
\end{equation}
The empty product is one. This follows from the scalar recurrence
$e_{j+1}\le\Lambda_j(1+\rho_j)e_j+\Lambda_j\rho_j\|H_j^0\|_F$.
Fallback positions can instead be handled with the absolute budget in Equation~\ref{eq:formal-mixed-branch-bound}.

For an analytical illustration, if the realized intervention norm is at most $\delta$ at each of $n$ successive layers and $\Lambda_j\le\Lambda$, Equation~\ref{eq:formal-forcing-bound} becomes $\delta\sum_{k=1}^{n}\Lambda^k$. It equals $n\delta$ when $\Lambda=1$, is at most $\delta\Lambda/(1-\Lambda)$ when $\Lambda<1$, and grows geometrically with $n$ when $\Lambda>1$. Thus the same local budget can have different downstream effects depending on the intervening maps. These conditional bounds explain why local relative scale and final-output change are separate quantities.

There is an analogous first-order description. Introduce a scalar perturbation amplitude $\tau$ with $R_j(\tau)=\tau V_j+o(\tau)$, $R_j(0)=0$, and differentiable frozen layers at the unperturbed trajectory. Differentiating $H_{j+1}(\tau)=\mathcal B_j(H_j(\tau)+R_j(\tau))$ gives
\begin{equation}
\dot H_L=\sum_{j=0}^{L-1}
J_{L-1}J_{L-2}\cdots J_j V_j,
\qquad J_j=J_{\mathcal B_j}(H_j^0).
\label{eq:formal-linear-response}
\end{equation}
Each interface direction is transported by the frozen Jacobians that follow it. The product acts on a low-dimensional injected vector, while different layers can inject through different receiving subspaces. This is the precise sense in which changing cross-layer communication can alter later computation without changing the pretrained operators themselves.

\subsection{The differentiated group-relative objective}\label{app:formal-objective}

For a retained group of $G=4$ completions, let $\mathcal T_i$ contain its scored completion positions and $T_i=\max\{1,|\mathcal T_i|\}$. Prompt and padding positions are excluded from $\mathcal T_i$. Let $V$ omit completions that are simultaneously truncated and unparseable; other incorrect completions are valid and receive binary reward zero. For $m=|V|>0$, write
\begin{equation}
\bar r=\frac1m\sum_{i\in V}r_i,\quad
\sigma^2=\frac1m\sum_{i\in V}(r_i-\bar r)^2,\quad
A_i=\begin{cases}
(r_i-\bar r)/(\sigma+10^{-6}),&i\in V,\ \sigma>10^{-6},\\
r_i-\bar r,&i\in V,\ \sigma\le10^{-6},\\
0,&i\notin V.
\end{cases}
\label{eq:formal-advantages}
\end{equation}
For $V=\varnothing$, all advantages and the policy term are zero. Centering gives $\sum_{i\in V}A_i=0$. For mixed binary rewards with $c$ correct among $m$ valid completions, $\sigma^2=(c/m)(1-c/m)$. Ignoring the displayed additive $10^{-6}$ only for this interpretation, correct and incorrect advantages are $\sqrt{(m-c)/c}$ and $-\sqrt{c/(m-c)}$ respectively. With the additive term retained their magnitudes are slightly smaller, and $|A_i|\le\sqrt{m-1}\le\sqrt3$. Uniform valid rewards produce zero advantages.

Put $\ell_{i,t}=\log\pi_\theta(y_{i,t}\mid x,y_{i,<t})$ and let $\ell^0_{i,t}$ be the corresponding adapter-disabled log probability. The differentiated surrogate uses a detached copy from the current evaluation,
\begin{equation}
\begin{aligned}
q_{i,t}&=\exp\!\left(\ell_{i,t}-\sg(\ell_{i,t})\right),\\
\mathcal L_{\rm policy}
&=-\frac1m\sum_{i\in V}\frac1{T_i}\sum_{t\in\mathcal T_i}
\min\!\left(q_{i,t}A_i,\operatorname{clip}(q_{i,t},0.8,1.2)A_i\right),\\
\mathcal L_{\rm KL}
&=\frac1{\sum_i|\mathcal T_i|}\sum_i\sum_{t\in\mathcal T_i}
\left[\exp(-\delta_{i,t})+\delta_{i,t}-1\right],
\quad\delta_{i,t}=\ell_{i,t}-\ell^0_{i,t}.
\end{aligned}
\label{eq:full-objective}
\end{equation}
The token-level KL reduction is defined for a group with at least one scored completion token. Every retained group receives one differentiable policy evaluation; accumulating gradients from multiple groups does not change the denominator construction.

\paragraph{Gradient equivalence to sequence-normalized REINFORCE.}
In the forward evaluation $q_{i,t}=1$, while its derivative is $\nabla_\theta q_{i,t}=\nabla_\theta\ell_{i,t}$. The value one lies strictly inside the clipping interval. In a neighborhood with the denominator held fixed, both arguments of the minimum therefore coincide and have the same derivative, irrespective of the sign of $A_i$. It follows that
\begin{equation}
\nabla_\theta\mathcal L_{\rm policy}
=\nabla_\theta\left[-\frac1m\sum_{i\in V}
\frac{\sg(A_i)}{T_i}\sum_{t\in\mathcal T_i}\ell_{i,t}\right].
\label{eq:reinforce-equivalence}
\end{equation}
This is equality of gradients for the retained observations and advantages. It is not equality of the two scalar loss values. When every valid completion has at least one scored token, the implemented policy value is $-m^{-1}\sum_{i\in V}A_i=0$ even when its gradient is nonzero. Its computational graph carries the gradient through the numerator despite cancellation of the evaluated values. A finite-difference check of this derivative must freeze the denominator at the expansion point; rebuilding a detached denominator separately at each perturbed point instead differentiates a constant-valued forward expression.

Sequence normalization gives each valid response one outer weight and averages its token scores before that weighting. The KL term uses a different reduction: each scored token has the same weight across all retained completions, including those excluded from $V$. Group standardization, validity filtering, sampling, and sequence normalization are therefore all part of the update estimator; Equation~\ref{eq:reinforce-equivalence} does not replace them by an unnormalized expected-reward objective.

\paragraph{Sampled divergence geometry.}
For $k(\delta)=e^{-\delta}+\delta-1$, $k\ge0$, $k'(\delta)=1-e^{-\delta}$, and $k''(\delta)=e^{-\delta}>0$. The minimum is at $\delta=0$, with expansion $k(\delta)=\delta^2/2+O(\delta^3)$. For a fixed prefix and full-support distributions $p,p_0$,
\begin{equation}
\E_{y\sim p}k\!\left(\log\frac{p(y)}{p_0(y)}\right)
=\sum_y p(y)\log\frac{p(y)}{p_0(y)}
=D_{\rm KL}(p\|p_0).
\label{eq:formal-kl-value-identity}
\end{equation}
The exponential term integrates to $\sum_y p_0(y)=1$. The optimizer, however, differentiates the sampled expression with the observed token fixed. Under a general sampling distribution $q$, that conditional expected derivative is
$\sum_y q(y)(1-p_0(y)/p(y))\nabla_\theta\log p(y)$.
This specifies the differentiated penalty independently of the value identity. Temperature, nucleus filtering, and group selection determine the actual sampling distribution, while the evaluated $p$ uses full-vocabulary, untempered logits.

\subsection{Routing regularization and its reduction}\label{app:formal-regularizer}

For the canonical 28 receiving layers, the complete auxiliary objective is
$\mathcal L=\mathcal L_{\rm policy}+0.02\mathcal L_{\rm KL}+0.01\Omega$, with
\begin{equation}
\Omega=\frac1{28}\sum_{\ell=4}^{31}
\left[\operatorname{mean}_{n,t,b}\left(a_{\ell,n,t,b}-\frac1{n_\ell}\right)^2
+0.05\operatorname{mean}_{n,t,k}c_{\ell,n,t,k}^{\,2}\right].
\label{eq:routing-regularizer}
\end{equation}
The first mean includes the $n_\ell$ visible sources; the second includes $r$ source-mixture coordinates. Both range over all $BT$ tensor positions. For one position with $n$ sources, define
$\psi(a)=n^{-1}\sum_b(a_b-1/n)^2$. Since $\sum_b a_b=1$,
\begin{equation}
\psi(a)=\frac1n\left(\|a\|_2^2-\frac1n\right),\qquad
0\le\psi(a)\le\frac{n-1}{n^2}.
\label{eq:formal-routing-penalty-bounds}
\end{equation}
The minimum is uniform allocation; the upper bound is reached at a simplex vertex and approached by dense softmax as logits separate. For $n=1$, the function and its logit gradient are identically zero. Thus at layers 4--7 the source-mixture magnitude term contributes while the routing-deviation term does not. This follows from source availability, not from a special early-layer regularization mask.

The source-coordinate mean also means that the maximum per-position deviation penalty depends on $n$: it is $1/4$ for two sources and $6/49$ for seven. Its gradient with respect to logits $z$ is
$\nabla_z\psi=\frac2n(\operatorname{diag}a-aa^\top)(a-\mathbf1/n)$.
The penalty acts on selection coefficients, while the source-mixture term acts on the mixture before $U_\ell$ and RMS calibration. Neither directly penalizes the norm of the final calibrated writeback projection.

To make the position reduction explicit, partition the $BT$ positions into prompt, completion, and padding sets $\mathcal P,\mathcal C,\mathcal Z$. For either scalar positionwise penalty $f$,
\begin{equation}
\frac1{BT}\sum_{n,t}f_{n,t}
=\sum_{Q\in\{\mathcal P,\mathcal C,\mathcal Z\}}
\frac{|Q|}{BT}\operatorname{mean}_{(n,t)\in Q}f_{n,t},
\label{eq:formal-position-mixture}
\end{equation}
where empty-set contributions are zero. This identity describes the implemented full-tensor objective: each position set contributes in proportion to its size. It also separates this reduction from the completion masks used by the two language-model loss terms.

\subsection{Bounded informative resampling and generation budgets}\label{app:formal-resampling}

For a fixed prompt and fixed model parameters during sampling, let $Y$ denote a group drawn from the generation procedure, and let $I(Y)$ indicate that at least one completion is parsed correct and at least one is not. Write $q=\Pr(I=1)$ and let $A$ be the maximum number of group attempts. The procedure retains the first mixed-correctness group; if no earlier attempt is mixed, it retains attempt $A$ regardless of its outcome. The configured values are $G=4$ and $A=4$ (three additional attempts).

Assume attempts are independent and identically distributed conditional on the prompt. With $N$ the number of attempts, the tail probabilities and expectation are
\begin{equation}
\Pr(N\ge j)=(1-q)^{j-1},\quad 1\le j\le A,
\qquad
\E N=\sum_{j=0}^{A-1}(1-q)^j
=\frac{1-(1-q)^A}{q}.
\label{eq:formal-retry-budget}
\end{equation}
The quotient has continuous limit $A$ at $q=0$. This follows because reaching attempt $j$ requires exactly that the preceding $j-1$ attempts fail the mixed-correctness test. The terminal probability is $\Pr(N=A)=(1-q)^{A-1}$, which includes both success and failure on the final attempt.

Let $\mu$ denote the original group law and $\mu_{\rm ret}$ the retained law. For any group event $E$ and $0<q<1$,
\begin{equation}
\begin{aligned}
\mu_{\rm ret}(E)
&=\frac{1-(1-q)^A}{q}\,\mu(E\cap I)
+(1-q)^{A-1}\mu(E\cap I^c),\\
\mu_{\rm ret}
&=[1-(1-q)^A]\,\mu(\cdot\mid I)
+(1-q)^A\,\mu(\cdot\mid I^c).
\end{aligned}
\label{eq:formal-retained-law}
\end{equation}
To prove the first line, sum $(1-q)^{j-1}\mu(E\cap I)$ over successful retention at each $j=1,\ldots,A$, then add $(1-q)^{A-1}\mu(E\cap I^c)$ for a failed final attempt. The second line groups these terms into normalized conditional distributions. Bounded retry increases the mass of mixed groups while preserving a nonzero failed-group mass $(1-q)^A$ whenever $q<1$.

The retry event precedes preference-validity filtering. Let $J$ denote a group whose \emph{valid} completions include both reward values. Then $J\subseteq I$, but equality need not hold: a group containing correct answers and truncated-unparseable failures is mixed before filtering and can become uniform afterwards. If $r=\Pr(J)$, Equation~\ref{eq:formal-retained-law} gives
\begin{equation}
\Pr_{\rm ret}(J)=r\frac{1-(1-q)^A}{q}.
\label{eq:formal-valid-informative-rate}
\end{equation}
For binary rewards, this is the probability of a retained group with nonzero advantages. It distinguishes the retry acceptance event from producing a nonconstant set of preference coefficients; the resulting parameter gradient also depends on the token-score derivatives.

\begin{table}[htbp]
\centering\small
\caption{\textbf{Analytical illustration of bounded retry.} Four independent Bernoulli-correctness completions per attempt, all preference-valid, with four attempts allowed. Values are computed from $q=1-p^4-(1-p)^4$.}
\label{tab:formal-retry-example}
\begin{tabular}{lrrrr}
\toprule
\rowcolor{paperhead}
Per-completion correctness $p$ & Mixed $q$ & Retained mixed & Expected attempts & Expected responses \\
\midrule
$0.1$ & $0.3438$ & $0.8146$ & $2.3694$ & $9.4774$ \\
$0.5$ & $0.8750$ & $0.9998$ & $1.1426$ & $4.5703$ \\
$0.9$ & $0.3438$ & $0.8146$ & $2.3694$ & $9.4774$ \\
\bottomrule
\end{tabular}
\end{table}

In the independent Bernoulli illustration, $q=1-p^G-(1-p)^G$. It is symmetric under $p\mapsto1-p$: both nearly always-correct and nearly always-incorrect prompts can consume additional attempts. The procedure therefore selects correctness diversity within groups, rather than preferring high accuracy by itself.

\paragraph{Responses, tokens, and retained updates.}
Every prompt yields one retained group but incurs between $G$ and $GA$ generated responses, with expectation $G\E N$. Let $C_j$ be the completion-token cost of attempt $j$, identically distributed jointly with its outcome. Reaching attempt $j$ depends only on earlier outcomes and is independent of $C_j$. Therefore
\begin{equation}
\E\!\left[\sum_{j=1}^{N}C_j\right]
=\sum_{j=1}^{A}\Pr(N\ge j)\E C_1
=\E N\,\E C_1.
\label{eq:formal-token-budget}
\end{equation}
This equality does not require the cost of an attempt to be independent of its own correctness outcome. The expected cost of discarded attempts, for $0<q<1$, is
\begin{equation}
\E C_{\rm discarded}
=\E[C_1\mid I^c](1-q)\sum_{j=0}^{A-2}(1-q)^j.
\label{eq:formal-discarded-cost}
\end{equation}
Only failed nonterminal attempts are discarded. Retained mixed and failed groups can have different length distributions, as made explicit by Equation~\ref{eq:formal-retained-law}.

For $D$ prompts and a cap of $M$ new tokens per completion, the deterministic bounds are $DG\le N_{\rm responses}\le DGA$ and $N_{\rm generated\ tokens}\le DGAM$. For one pass over $D=7{,}473$ prompts with $G=A=4$, the response budget is 29,892--119,568 completions and the token bound is $119{,}568M$. The number of differentiated retained groups remains 7,473. These count generation separately from teacher-forced scoring and from gradient accumulation. They make precise why prompt count, optimizer updates, and generated-token budget characterize different aspects of the same RL procedure.

\clearpage
\stopcontents[trmspecification]
\resumecontents[trmresults]
\section{Complete empirical results and design analysis}\label{app:results-atlas}\label{app:diagnostics}
This appendix develops the evidence for parameter-efficient computation reuse. Method comparisons pair reasoning quality with trainable allocation; component comparisons examine how that capacity is organized; writeback diagnostics and configuration profiles characterize the resulting task responses. Generative scores are means $\pm$ sample standard deviations (SD) across three evaluation seeds, 42, 43, and 44. Bold/underlined means identify the best/second-best values within each column, including ties; they describe the ordering of means. Lower values are preferred for resource costs and perplexity. Historical configurations retain their $\dagger$ estimate attribute, and $\ddagger$ marks original non-generative point references. The complete sample variances appear in Appendix~\ref{app:evaluation-variance}.

\subsection{Method comparisons across mathematics and search}
For each evaluation seed, AIME mean averages the two annual scores and MathAvg gives one-third weight to GSM8K, MATH-500, and AIME mean. The reported aggregate SD is computed across those seed-level aggregates. Differences likewise pair the same evaluation-seed indices before taking their mean and SD. Table~\ref{tab:math-supplement} completes the annual and aggregate profiles in Table~\ref{tab:math}; Table~\ref{tab:search-complete} retains the complete search comparison.

\begin{table}[!htb]
\centering\footnotesize
\setlength{\tabcolsep}{2.6pt}
\caption{\textbf{Annual mathematics results and aggregate contrasts.} Accuracy and MathAvg are percentages; SD and $\Delta$ are percentage points. $\Delta$ is the paired MathAvg difference from full-parameter GRPO. GSM8K, MATH-500, and parameter allocations appear in Table~\ref{tab:math}.}
\label{tab:math-supplement}
\begin{tabular}{lrrrrr}
\toprule
\rowcolor{paperhead}
Method & AIME24 & AIME25 & AIME mean & MathAvg & $\Delta$ vs GRPO \\
\midrule
Frozen base & \stat{30.00}{5.77} & \stat{33.33}{3.33} & \stat{31.67}{3.33} & \stat{65.33}{0.89} & \stat{-8.46}{0.95} \\
Full-parameter GRPO & \stat{44.44}{6.94} & \stat{44.44}{3.85} & \stat{44.44}{5.36} & \stat{73.79}{1.83} & \stat{0.00}{0.00} \\
RFT & \secondstat{51.11}{10.18} & \stat{41.11}{13.47} & \stat{46.11}{2.55} & \stat{73.64}{1.33} & \stat{-0.15}{2.54} \\
LoRA-r2 + GRPO & \stat{43.33}{17.32} & \stat{41.11}{10.72} & \stat{42.22}{12.73} & \stat{73.42}{4.28} & \stat{-0.37}{2.67} \\
LoRA-r16 + GRPO & \stat{50.00}{5.77} & \stat{41.11}{6.94} & \stat{45.56}{6.31} & \stat{74.83}{1.41} & \stat{+1.04}{2.69} \\
LoRA-MoE + RO-GRPO & \stat{41.11}{5.09} & \stat{47.78}{5.09} & \stat{44.44}{2.55} & \stat{75.74}{1.47} & \stat{+1.95}{0.41} \\
LoRA + S-GRPO & \secondstat{51.11}{19.53} & \stat{46.67}{3.33} & \secondstat{48.89}{8.22} & \stat{75.95}{2.43} & \stat{+2.16}{2.30} \\
LoRA-r16, matched retries & \stat{45.56}{7.70} & \secondstat{51.11}{12.62} & \stat{48.33}{6.01} & \secondstat{77.28}{1.95} & \secondstat{+3.49}{3.76} \\
\rowcolor{trmlight}\trm\ + GRPO & \beststat{63.33}{6.67} & \beststat{57.78}{1.92} & \beststat{60.56}{3.47} & \beststat{83.64}{1.16} & \beststat{+9.85}{2.93} \\
\bottomrule
\end{tabular}
\end{table}

\begin{table}[!htb]
\centering\small
\setlength{\tabcolsep}{3pt}
\caption{\textbf{Complete agentic search comparison.} BrowseComp Plus uses token-F1 multiplied by 100; ASearch Test denotes the ASearcher held-out validation score (0--100). Appendix~\ref{app:search-protocol} details the scorer and default penalty factors. $\Delta$ is the paired difference from full-parameter GRPO; all entries are mean $\pm$ SD in the corresponding metric\textquotesingle s points.}
\label{tab:search-complete}
\begin{tabular}{lrrrr}
\toprule
\rowcolor{paperhead}
Method & BrowseComp F1 & $\Delta$ vs GRPO & ASearch & $\Delta$ vs GRPO \\
\midrule
Frozen base & \stat{5.11}{0.30} & \stat{-28.12}{1.25} & \stat{59.76}{0.41} & \stat{-10.19}{0.56} \\
Full-parameter GRPO & \stat{33.23}{1.54} & \stat{0.00}{0.00} & \stat{69.94}{0.53} & \stat{0.00}{0.00} \\
RFT & \stat{31.20}{0.75} & \stat{-2.03}{2.07} & \stat{65.81}{0.92} & \stat{-4.14}{1.18} \\
LoRA-r2 + GRPO & \stat{32.98}{0.56} & \stat{-0.25}{1.25} & \stat{69.19}{0.59} & \stat{-0.75}{1.01} \\
LoRA-r16 + GRPO & \stat{36.06}{0.56} & \stat{+2.83}{1.80} & \stat{70.59}{0.18} & \stat{+0.64}{0.64} \\
LoRA-MoE + RO-GRPO & \beststat{37.00}{1.17} & \beststat{+3.77}{2.56} & \stat{71.46}{0.55} & \stat{+1.52}{0.65} \\
LoRA + S-GRPO & \stat{35.50}{0.57} & \stat{+2.27}{1.58} & \stat{70.67}{0.19} & \stat{+0.73}{0.34} \\
LoRA-r16, matched retries & \stat{36.90}{1.13} & \stat{+3.67}{1.32} & \secondstat{72.80}{0.75} & \secondstat{+2.86}{0.63} \\
\rowcolor{trmlight}\trm\ + GRPO & \secondstat{36.98}{1.16} & \secondstat{+3.75}{0.91} & \beststat{73.98}{0.74} & \beststat{+4.04}{1.27} \\
\bottomrule
\end{tabular}
\end{table}

Against full-parameter GRPO, \trm\ improves MathAvg by $9.85\pm2.93$ points with a trainable allocation equal to 0.466\% of the backbone. At a similar adapter budget, matched-retry LoRA records $77.28\pm1.95$ versus \trm's $83.64\pm1.16$, a 6.36-point difference in means. The corresponding search means are 36.90 versus 36.98 on BrowseComp Plus and 72.80 versus 73.98 on ASearch. LoRA-MoE records 37.00 BrowseComp F1 with 173.112M trainable parameters; \trm\ records 36.98 with 41.730M. Together these profiles establish a favorable quality--allocation tradeoff and the strongest ASearch mean in the comparison.

\FloatBarrier
\subsection{Complete component ablations}\label{app:component-results}
Table~\ref{tab:ablation-complete} retains all four task scores and the paired aggregate contrast for each component change. The no-state alternative uses a capacity-matched MLP, allowing the controller\textquotesingle s role to be assessed alongside its allocation.

\begin{table}[!htb]
\centering\footnotesize
\setlength{\tabcolsep}{1.5pt}
\caption{\textbf{Complete task-level component ablations.} Entries are mean $\pm$ SD across evaluation seeds. $\Delta$ is each variant minus full \trm\ in MathAvg percentage points.}
\label{tab:ablation-complete}
\begin{tabular}{lrrrrrr}
\toprule
\rowcolor{paperhead}
Configuration & GSM8K & MATH & AIME24 & AIME25 & MathAvg & $\Delta$ \\
\midrule
\rowcolor{trmlight}Full \trm & \beststat{97.62}{0.79} & \beststat{92.73}{0.70} & \beststat{63.33}{6.67} & \beststat{57.78}{1.92} & \beststat{83.64}{1.16} & \beststat{0.00}{0.00} \\
No cross-layer retrieval & \stat{96.26}{0.59} & \stat{79.40}{1.56} & \stat{14.44}{3.85} & \secondstat{51.11}{15.03} & \stat{69.48}{2.36} & \stat{-14.16}{1.78} \\
Hidden-state memory & \secondstat{96.79}{0.53} & \stat{77.80}{1.31} & \stat{37.78}{10.72} & \stat{17.78}{5.09} & \stat{67.46}{2.19} & \stat{-16.18}{2.96} \\
Learned static routing & \stat{96.01}{0.68} & \stat{77.27}{0.81} & \stat{26.67}{5.77} & \stat{16.67}{6.67} & \stat{64.98}{1.71} & \stat{-18.66}{2.78} \\
No state (matched MLP) & \stat{95.02}{0.75} & \secondstat{83.93}{1.92} & \secondstat{50.00}{17.64} & \stat{28.89}{10.18} & \secondstat{72.80}{4.10} & \secondstat{-10.84}{3.09} \\
No layer/block identity & \stat{94.29}{0.42} & \stat{66.13}{1.80} & \stat{11.11}{10.72} & \stat{16.67}{3.33} & \stat{58.10}{1.06} & \stat{-25.53}{0.56} \\
No RMS alignment & \stat{94.57}{0.64} & \stat{67.93}{3.80} & \stat{3.33}{3.33} & \stat{10.00}{3.33} & \stat{56.39}{0.41} & \stat{-27.25}{1.24} \\
No token-conditioned gate & \stat{96.26}{0.42} & \stat{66.73}{1.81} & \stat{34.44}{8.39} & \stat{22.22}{11.71} & \stat{63.78}{2.61} & \stat{-19.86}{1.47} \\
No informative retries & \stat{94.39}{0.55} & \stat{74.87}{2.55} & \stat{14.44}{1.92} & \stat{22.22}{3.85} & \stat{62.53}{1.20} & \stat{-21.11}{2.26} \\
\bottomrule
\end{tabular}
\end{table}

\paragraph{Computation content and addressed retrieval.}
Full \trm\ reaches $83.64\pm1.16$ MathAvg. Removing cross-layer retrieval gives $69.48\pm2.36$, and replacing completed block changes with hidden states gives $67.46\pm2.19$. These changes distinguish access to earlier computations from the representation stored in memory. The corresponding MATH-500 means are 92.73, 79.40, and 77.80. Source/layer identity identifies the origin and destination of each exchange; its removal has a paired MathAvg difference of $-25.53\pm0.56$ points.

\paragraph{History, selection, and relative influence.}
The matched MLP records $72.80\pm4.10$ MathAvg and learned static routing records $64.98\pm1.71$, supporting the organization of capacity into recurrent depth history and query-dependent selection. Removing RMS alignment gives a paired difference of $-27.25\pm1.24$ points, while removing the token-conditioned gate gives $-19.86\pm1.47$. Calibration sets the relative intervention scale; the signed gate adjusts that influence to the current token. The complete interface learns which computations to reuse and how strongly to incorporate them.

\subsection{Parameter and training-cost profiles of the ablations}
Table~\ref{tab:ablation-resources} pairs the same configurations with exact allocated parameter counts and their original GPU-hour records. The common estimated optimizer-step count separates update count from response generation and the computation performed within an update.

\begin{table}[!htb]
\centering\small
\setlength{\tabcolsep}{3.5pt}
\caption{\textbf{Resources for the component study.} Parameter counts are exact allocations; steps retain their estimated status. GPU-hours retain their original training-cost records, separate from the evaluation-seed statistics.}
\label{tab:ablation-resources}
\begin{tabular}{lrrr}
\toprule
\rowcolor{paperhead}
Configuration & Parameters & Steps (est.) & GPU-hours \\
\midrule
\rowcolor{trmlight}Full \trm & 41,730,332 & 3,737 & 121.23 \\
No cross-layer retrieval & \underline{41,648,924} & 3,737 & 110.97 \\
Hidden-state memory & 41,730,332 & 3,737 & 122.42 \\
Learned static routing & \textbf{40,535,324} & 3,737 & 114.90 \\
No state (matched MLP) & 41,731,978 & 3,737 & \underline{69.25} \\
No layer/block identity & 41,729,052 & 3,737 & 77.34 \\
No RMS alignment & 41,730,332 & 3,737 & 117.70 \\
No token-conditioned gate & 41,725,980 & 3,737 & 123.89 \\
No informative retries & 41,730,332 & 3,737 & \textbf{37.92} \\
\bottomrule
\end{tabular}
\end{table}

The MLP allocates 41,731,978 parameters, 1,646 more than full \trm's 41,730,332. Its paired MathAvg difference is $-10.84\pm3.09$ points. Hidden-state memory and the no-RMS variant retain exactly the full allocation; removing the token-conditioned gate changes it by 4,352 parameters. These comparisons link accuracy to the organization of trainable capacity.

Informative retries address a complementary training dimension. The no-retry configuration records 37.92 GPU-hours and $62.53\pm1.20$ MathAvg, while full \trm\ records 121.23 GPU-hours and $83.64\pm1.16$. Their paired MathAvg difference is $21.11\pm2.26$ points in favor of the full configuration. Retry generation changes the completions available to each prompt\textquotesingle s update while retaining the same estimated optimizer-step count, connecting the empirical comparison to the correctness diversity in Appendix~\ref{app:formal}.

\FloatBarrier
\subsection{Configuration landscape and annual task profiles}
The historical exploration varies writeback rules, controller structure, and optimization settings. Table~\ref{tab:configuration-profiles} preserves every annual score from its initialization and structural branch, while Table~\ref{tab:writeback-complete} records the complete writeback sweep. The historical $\dagger$ attribute applies to GSM8K, MATH-500, and their aggregates; evaluation-seed dispersion does not change that provenance. Full \trm\ is the final-model reference.

\begin{table}[!htb]
\centering\footnotesize
\setlength{\tabcolsep}{2.6pt}
\caption{\textbf{Complete historical task profiles and final reference.} Generative entries are mean $\pm$ SD; $\dagger$ retains the historical IID-estimate attribute. The initialization and local/controller variants use the standard-gate branch. Full \trm\ is the final configuration from Table~\ref{tab:math}.}
\label{tab:configuration-profiles}\label{tab:atlas-matrix}\label{tab:other}
\begin{tabular}{lrrrrr}
\toprule
\rowcolor{paperhead}
Configuration & GSM8K & MATH-500 & AIME24 & AIME25 & MathAvg \\
\midrule
Standard gate$^\dagger$ & \stat{87.57}{0.87} & \stat{84.87}{3.29} & \stat{36.67}{0.00} & \stat{27.78}{10.72} & \stat{68.22}{2.83} \\
Standard gate, no KL$^\dagger$ & \stat{89.03}{0.58} & \stat{71.67}{4.02} & \stat{13.33}{5.77} & \stat{4.44}{1.92} & \stat{56.53}{1.72} \\
Local16$^\dagger$ & \stat{90.75}{1.12} & \stat{82.00}{1.00} & \stat{33.33}{3.33} & \stat{35.56}{18.36} & \stat{69.07}{2.74} \\
Local64$^\dagger$ & \stat{89.54}{0.57} & \stat{71.07}{2.57} & \stat{5.56}{5.09} & \stat{11.11}{5.09} & \stat{56.31}{2.48} \\
Wide512$^\dagger$ & \stat{87.57}{1.18} & \stat{84.60}{1.93} & \stat{24.44}{10.72} & \stat{27.78}{5.09} & \stat{66.09}{2.85} \\
Zero initialization$^\dagger$ & \stat{89.03}{0.35} & \stat{83.87}{1.81} & \stat{51.11}{10.72} & \stat{30.00}{6.67} & \stat{71.15}{2.29} \\
Orthogonal initialization$^\dagger$ & \stat{90.83}{0.08} & \stat{82.93}{1.72} & \stat{42.22}{5.09} & \stat{28.89}{1.92} & \stat{69.77}{0.38} \\
Top2 composite$^\dagger$ & \secondstat{96.64}{0.29} & \secondstat{87.20}{1.06} & \secondstat{58.89}{10.72} & \secondstat{41.11}{9.62} & \secondstat{77.95}{2.12} \\
\rowcolor{trmlight}Full \trm & \beststat{97.62}{0.79} & \beststat{92.73}{0.70} & \beststat{63.33}{6.67} & \beststat{57.78}{1.92} & \beststat{83.64}{1.16} \\
\bottomrule
\end{tabular}
\end{table}

Full \trm\ has the largest mean on all four tasks in this panel. Relative to Top2, the approximately 5.69-point MathAvg difference comprises 0.33 point from GSM8K, 1.84 from MATH-500, and 3.52 from pooled AIME. The annual means are 63.33 and 57.78 for full \trm, compared with 58.89 and 41.11 for Top2. Zero initialization has the largest annual mean gap in this panel, 21.11 points. These annual gaps compare two task sets; their seed-paired SDs are not among the reported statistics.

Zero and orthogonal initialization modify the standard-gate branch; the orthogonal gain is 0.064, whereas RMS configurations use gain 1. Local16 and Local64 change the additional token-local branch rank; Wide512 changes controller dimensions; Top2 denotes the composite sparse-routing configuration specified in Appendix~\ref{app:variants}.

\begin{figure}[!htb]
\centering
\includegraphics[width=.96\linewidth]{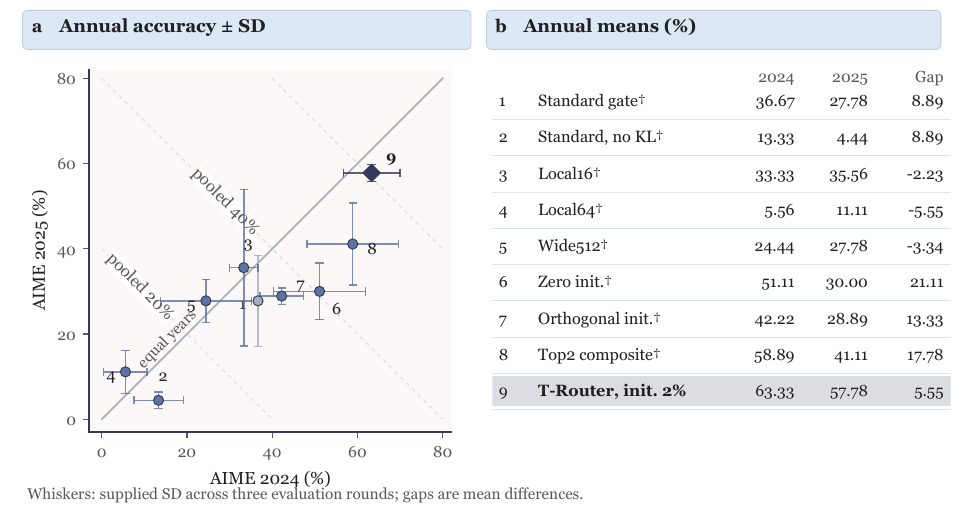}
\caption{\textbf{Annual competition-mathematics profiles.} The two annual means locate each configuration; whiskers show supplied SDs, and numerical summaries give annual means and their gaps. Pooled means average AIME24 and AIME25, and annual gaps are differences of those means. Each annual set contains 30 questions per evaluation seed.}
\label{fig:atlas-aime}
\end{figure}

\subsection{Benchmark contributions to aggregate differences}
MathAvg admits an additive decomposition. For mean family score $s_j(c)$ and standard-gate reference $c_0$,
\begin{equation}
\Delta\mathrm{MathAvg}(c,c_0)=\frac{s_{\rm G}(c)-s_{\rm G}(c_0)}3+\frac{s_{\rm M}(c)-s_{\rm M}(c_0)}3+\frac{s_{\rm A}(c)-s_{\rm A}(c_0)}3.
\label{eq:atlas-contributions}
\end{equation}
Figure~\ref{fig:atlas-contributions} applies this identity to the displayed mean profiles. These decompositions and the later reweighting analyses use supplied rounded means; small rounding differences from separately reported aggregates can therefore occur.

\begin{figure}[!htb]
\centering
\includegraphics[width=\linewidth]{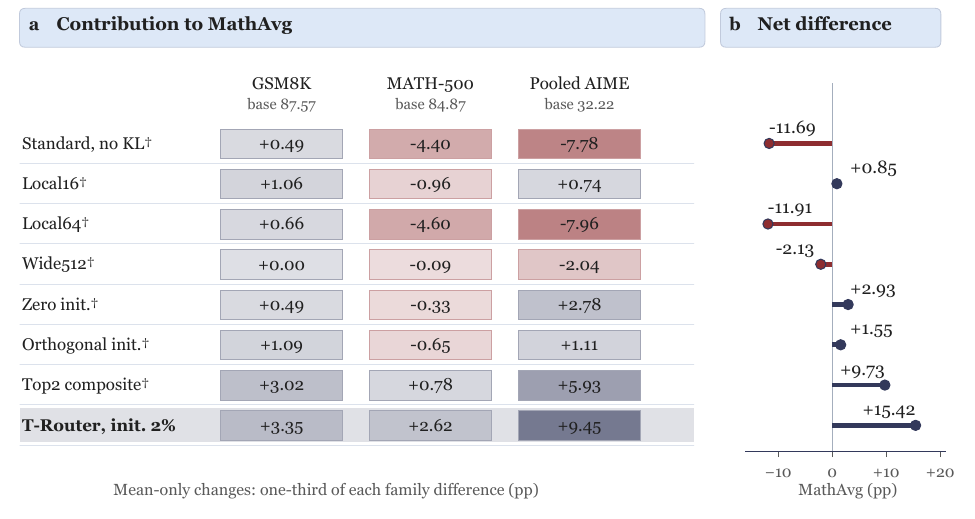}
\caption{\textbf{Where aggregate mean differences occur.} Each family contributes one-third of its signed difference from the standard gate. Navy/dark red indicate positive/negative contributions; the net column sums those mean contrasts. These are analytical decompositions of the same reported evaluations.}
\label{fig:atlas-contributions}
\end{figure}

Full \trm's approximately 15.42-point difference from the standard gate comprises 3.35 points from GSM8K, 2.62 from MATH-500, and 9.45 from pooled AIME. Top2\textquotesingle s approximately 9.73-point difference comprises 3.02, 0.78, and 5.93 points. Both distribute gains across all three mathematical families, with the full design adding a stronger MATH-500 and AIME profile.

\begingroup\raggedright
Local16 contributes approximately $(+1.06,-0.96,+0.74)$ points, while Local64 contributes $(+0.66,-4.60,-7.96)$. The no-KL branch contributes $(+0.49,-4.40,-7.78)$. Zero and orthogonal initialization produce net mean differences of approximately $+2.93$ and $+1.55$, respectively. These task-level views characterize the effect of complete settings alongside their aggregate order.\par\endgroup

\subsection{Writeback configurations and task response}
The writeback study connects relative intervention scale to the quality of the reused computation. The standard gate learns a coefficient without RMS alignment; fixed and learned RMS configurations express writeback on a common relative-scale basis. Table~\ref{tab:writeback-complete} contains every reported family score, aggregate, and endpoint diagnostic.

\begin{table}[!htb]
\centering\footnotesize
\setlength{\tabcolsep}{2pt}
\caption{\textbf{Complete writeback exploration and final-model reference.} Scores are mean $\pm$ SD; $\dagger$ retains historical estimates. Final WB and tail median are original percentage-valued endpoint diagnostics. A dash denotes an unavailable trailing median.}
\label{tab:writeback-complete}\label{tab:atlas-scale}
\begin{tabular}{lrrrrrr}
\toprule
\rowcolor{paperhead}
Configuration & GSM8K & MATH & AIME mean & MathAvg & Final WB & Tail \\
\midrule
Standard gate$^\dagger$ & \stat{87.57}{0.87} & \stat{84.87}{3.29} & \stat{32.22}{5.36} & \stat{68.22}{2.83} & 0.1254 & --- \\
Fixed RMS 0.5\%$^\dagger$ & \stat{87.47}{0.99} & \stat{76.20}{1.59} & \stat{35.56}{3.47} & \stat{66.41}{1.69} & 0.5005 & --- \\
Fixed RMS 1\%$^\dagger$ & \stat{89.31}{0.67} & \stat{79.67}{3.11} & \stat{40.00}{5.00} & \stat{69.66}{2.34} & 1.0010 & --- \\
Fixed RMS 2\%$^\dagger$ & \stat{96.99}{0.56} & \secondstat{87.60}{2.11} & \stat{47.22}{2.55} & \stat{77.27}{1.16} & 2.0020 & --- \\
Learned RMS, init 0.5\%$^\dagger$ & \stat{96.99}{0.29} & \stat{84.67}{0.99} & \secondstat{55.00}{8.33} & \stat{78.89}{2.39} & 3.4987 & 3.7596 \\
Learned RMS, init 1\%$^\dagger$ & \beststat{97.83}{0.46} & \stat{84.07}{1.67} & \secondstat{55.00}{12.02} & \secondstat{78.96}{4.40} & 3.6535 & 3.8762 \\
\rowcolor{trmlight}Full \trm & \secondstat{97.62}{0.79} & \beststat{92.73}{0.70} & \beststat{60.56}{3.47} & \beststat{83.64}{1.16} & 4.0211 & 4.2065 \\
\bottomrule
\end{tabular}
\end{table}

\begin{figure}[!htb]
\centering
\includegraphics[width=\linewidth]{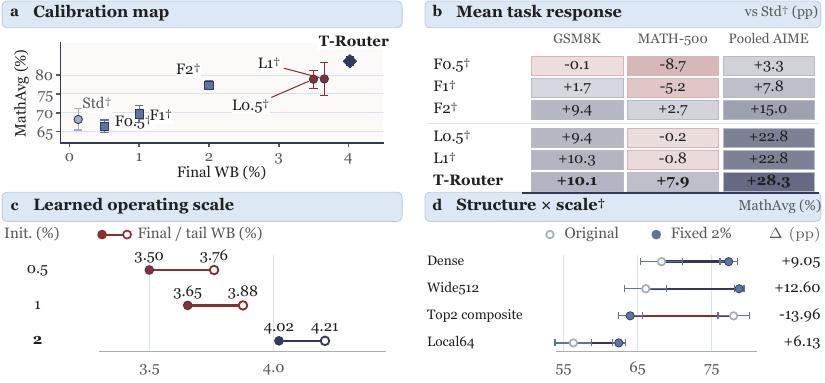}
\caption{\textbf{Writeback scale, task response, and structure.} (a) MathAvg versus final relative writeback. (b) Mean family-score differences from the standard gate. F/L denote fixed/learned RMS, with coefficients or initializations in percent. (c) Original final and trailing-median diagnostics. (d) Structure-dependent mean changes under fixed 2\% RMS. Whiskers in (a,d) show evaluation-seed SD. Tables~\ref{tab:writeback-complete} and~\ref{tab:interactions} give the complete statistics.}
\label{fig:writeback}
\end{figure}

Full \trm\ records $83.64\pm1.16$ MathAvg at 4.0211\% final writeback, combining a 92.73 MATH-500 mean with 60.56 pooled AIME. Learned RMS initialized at 0.5\% and 1\% records MathAvg means of 78.89 and 78.96, with pooled AIME means of 55.00 for both. These task profiles connect writeback scale to the computations delivered to the receiving layer.

\subsection{Writeback coefficients and endpoint diagnostics}
The standard gate has a 20\% bound and 5\% initial coefficient, with small-normal writeback projections. Fixed RMS uses gain-1 orthogonal projections and constant coefficients of 0.5\%, 1\%, or 2\%. Learned RMS uses the same gain-1 projection initialization, a 5\% gate bound, and initial coefficients of 0.5\%, 1\%, or 2\%. A fixed coefficient retains input-dependent source routing.

\begin{figure}[!htb]
\centering
\includegraphics[width=\linewidth]{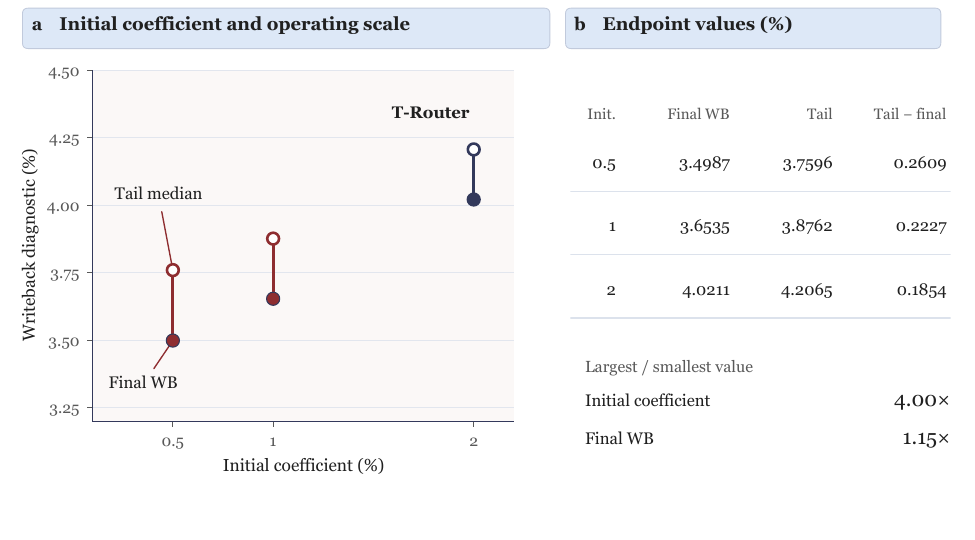}
\caption{\textbf{Writeback endpoint summaries.} Initial coefficients are paired with Final WB (filled circles) and trailing-median WB (open circles). The table gives the original scalar diagnostics, their differences, and the largest-to-smallest initial and final ratios. Final WB is the latest routed layer\textquotesingle s whole-tensor writeback-to-residual norm ratio.}
\label{fig:atlas-endpoints}
\end{figure}

For the 0.5\%, 1\%, and 2\% learned initializations, tail medians are 3.7596\%, 3.8762\%, and 4.2065\%, exceeding final diagnostics by 0.2609, 0.2227, and 0.1854 points. The final scales lie closer together than the initial coefficients. Across the broader comparison, final scale alone does not order accuracy: fixed RMS 0.5\% has a larger final diagnostic than the standard gate, while their MathAvg means are 66.41 and 68.22. Retrieval direction, the receiving representation, and scaling jointly determine the resulting computation.

\FloatBarrier
\subsection{Structure--writeback combinations}\label{app:variants}
The interaction study pairs four structures with their original writeback and fixed 2\% RMS. Table~\ref{tab:interactions} preserves the mean and SD of each paired original-to-fixed change, retaining each structure\textquotesingle s routing and state computation.

\begin{table}[!htb]
\centering\small
\setlength{\tabcolsep}{4pt}
\caption{\textbf{Structure--writeback combinations.} MathAvg is in percent; SD and paired changes are in percentage points; $\dagger$ retains the historical estimate attribute. Change is fixed RMS minus original writeback, paired by evaluation-seed index.}
\label{tab:interactions}
\begin{tabular}{lrrr}
\toprule
\rowcolor{paperhead}
Structure & Original writeback & Fixed RMS 2\% & Paired change \\
\midrule
Dense, standard gate$^\dagger$ & \secondstat{68.22}{2.83} & \secondstat{77.27}{1.16} & \secondstat{+9.05}{2.26} \\
Wide512$^\dagger$ & \stat{66.09}{2.85} & \beststat{78.69}{0.66} & \beststat{+12.60}{3.41} \\
Top2 composite$^\dagger$ & \beststat{77.95}{2.12} & \stat{63.98}{1.62} & \stat{-13.96}{1.95} \\
Local64$^\dagger$ & \stat{56.31}{2.48} & \stat{62.44}{0.90} & \stat{+6.13}{1.62} \\
\bottomrule
\end{tabular}
\end{table}

\begin{figure}[!htb]
\centering
\includegraphics[width=\linewidth]{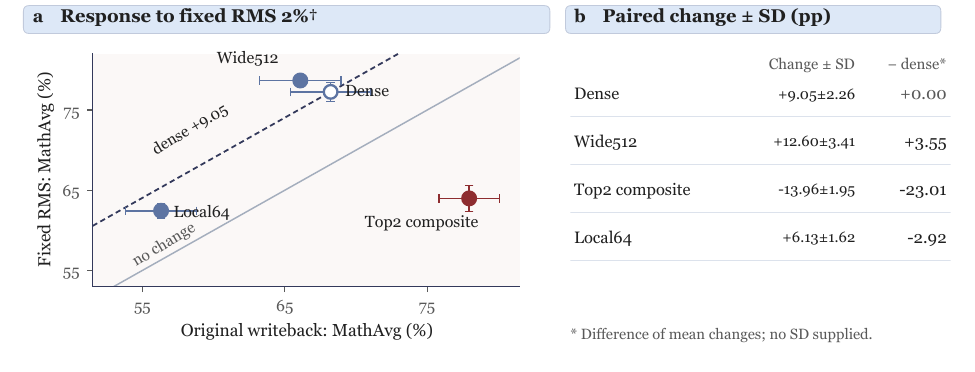}
\caption{\textbf{Structure-dependent response to fixed RMS writeback.} Original and fixed-RMS scores in (a) show means and evaluation SDs. The dashed reference adds the dense row\textquotesingle s 9.05-point change. The paired changes use supplied SDs; differences between those changes are descriptive contrasts of means.}
\label{fig:atlas-interactions}
\end{figure}

Dense routing and Wide512 change by $+9.05\pm2.26$ and $+12.60\pm3.41$ points; the difference between these means is $+3.55$. Local64 changes by $+6.13\pm1.62$, or $-2.92$ relative to the dense response. Top2 changes by $-13.96\pm1.95$, giving a $-23.01$ mean contrast relative to dense routing. The signs show that a configuration\textquotesingle s original quality and its response to fixed calibration describe distinct aspects of the design space.

\paragraph{Meaning of the structure names.}
Local16 and Local64 add a token-local branch of rank 16 or 64, a controller-conditioned sigmoid gate initialized at 0.2, and a branch learning-rate multiplier of two; cross-layer memory rank remains 256. Wide512 sets slot and controller hidden dimensions to 512, with eight slots and memory rank 256. Top2 restricts source and slot selection to two entries, RMS-normalizes stored deltas, sets routing temperature to 0.7, and uses routing-sharpness, balance, and slot-diversity terms. Its result therefore characterizes that complete configuration.

Wide512 with fixed RMS records $78.69\pm0.66$ MathAvg with 45.546M allocated parameters, versus $83.64\pm1.16$ and 41.730M for the selected learned-RMS design. The 3.816M allocation difference accompanies a 4.95-point difference in mean MathAvg. Allocation remains distinct from active memory and arithmetic during execution.

\FloatBarrier
\subsection{Auxiliary task profiles}
The auxiliary tasks describe narrative reasoning (MuSR), demanding multiple-choice questions (GPQA-D), instruction following (IFEval), and language-model likelihood (WikiText-2). IFEval strict@50 uses 50 prompts per evaluation seed and requires satisfying every specified constraint for a prompt.

\begin{table}[!htb]
\centering\small
\setlength{\tabcolsep}{3.5pt}
\caption{\textbf{Complete auxiliary task profiles.} IFEval is mean $\pm$ sample SD over evaluation seeds. $\ddagger$ marks original non-generative point references for MuSR, GPQA-D, and perplexity; no evaluation-seed dispersion is assigned to those entries. Lower perplexity is preferred; dashes denote unavailable entries.}
\label{tab:main-diagnostics}
\begin{tabular}{lrrrr}
\toprule
\rowcolor{paperhead}
Configuration & MuSR & GPQA-D & IFEval strict@50 & WikiText-2 PPL \\
\midrule
Frozen base & --- & --- & --- & --- \\
Full-parameter GRPO & \textbf{63.00}$^\ddagger$ & \underline{40.00}$^\ddagger$ & --- & --- \\
Standard gate & \textbf{63.00}$^\ddagger$ & \textbf{41.00}$^\ddagger$ & \stat{31.33}{3.06} & \textbf{7.0391}$^\ddagger$ \\
Fixed RMS 2\% & \underline{61.00}$^\ddagger$ & 39.00$^\ddagger$ & \secondstat{36.00}{4.00} & 7.0890$^\ddagger$ \\
Learned RMS, init 1\% & \underline{61.00}$^\ddagger$ & \underline{40.00}$^\ddagger$ & \stat{34.00}{7.21} & \underline{7.0758}$^\ddagger$ \\
\rowcolor{trmlight}Full \trm & \underline{61.00}$^\ddagger$ & 38.00$^\ddagger$ & \beststat{38.67}{4.16} & 7.1185$^\ddagger$ \\
\bottomrule
\end{tabular}
\end{table}

\begin{figure}[!htb]
\centering
\includegraphics[width=\linewidth]{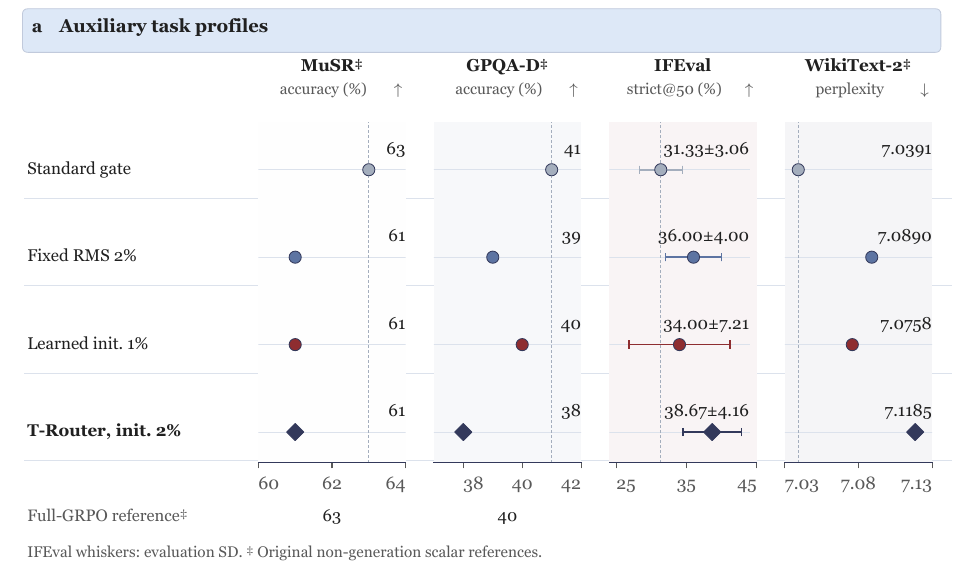}
\caption{\textbf{Task-specific diagnostic scales.} Each task retains its native units. IFEval means carry evaluation-seed SDs; $\ddagger$ point references retain their original non-generative scoring status. The full-GRPO reference is available for the two multiple-choice tasks.}
\label{fig:atlas-auxiliary}
\end{figure}

Full \trm\ records the highest IFEval mean, $38.67\pm4.16$, compared with $36.00\pm4.00$ for fixed RMS 2\%, $34.00\pm7.21$ for learned initialization at 1\%, and $31.33\pm3.06$ for the standard gate. The three RMS settings share the MuSR point reference of 61.00, with GPQA-D references between 38.00 and 40.00. Their WikiText-2 perplexities are 7.0890, 7.0758, and 7.1185 for fixed 2\%, learned 1\%, and full \trm, respectively. These values retain the separate meanings of accuracy, constraint satisfaction, and likelihood.

\FloatBarrier
\subsection{Sensitivity to the choice of family weights}
MathAvg assigns equal weight to the three mathematical families. To characterize alternative priorities, reaggregate the same mean profiles with $w_{\rm G},w_{\rm M},w_{\rm A}\ge0$ and $w_{\rm G}+w_{\rm M}+w_{\rm A}=1$:
\begin{equation}
S_w(c)=w_{\rm G}s_{\rm G}(c)+w_{\rm M}s_{\rm M}(c)+w_{\rm A}s_{\rm A}(c).
\label{eq:atlas-weighted}
\end{equation}

\begin{figure}[!htb]
\centering
\includegraphics[width=\linewidth]{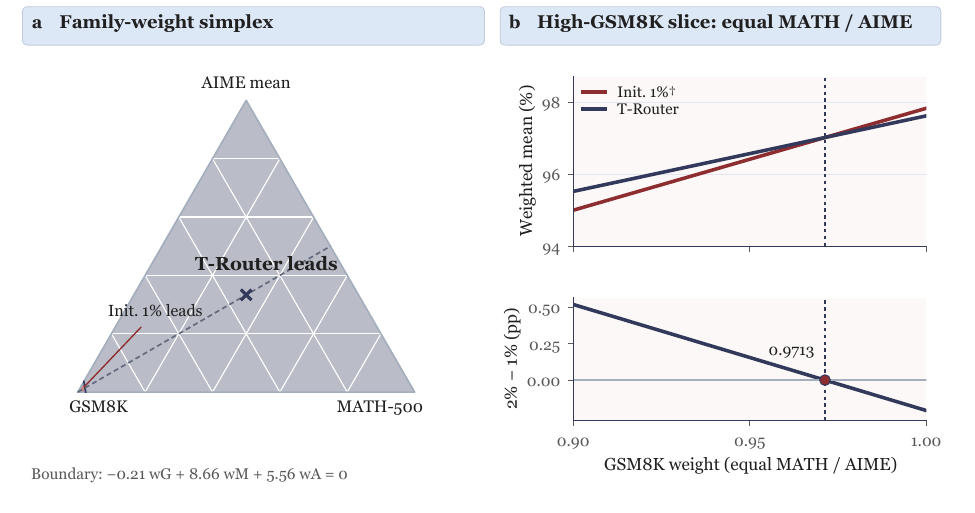}
\caption{\textbf{Analytical sensitivity to benchmark-family weights.} The simplex shows the largest weighted mean among 14 complete profiles. The one-dimensional slice gives equal weight to MATH-500 and AIME and varies GSM8K weight; the crossover is near $w_{\rm G}=0.9713$. Scores are analytical reaggregations of the supplied means using Equation~\ref{eq:atlas-weighted}.}
\label{fig:atlas-weights}
\end{figure}

Full \trm\ has mean profile $(97.62,92.73,60.56)$, versus $(97.83,84.07,55.00)$ for learned RMS initialized at 1\%. Therefore,
\begin{equation}
S_w(\text{init 2\%})-S_w(\text{init 1\%})=-0.21w_{\rm G}+8.66w_{\rm M}+5.56w_{\rm A}.
\label{eq:atlas-boundary}
\end{equation}
These two profiles define the upper envelope: full \trm\ strictly exceeds every other historical profile in all three family means. The 1\% initialization has the largest GSM8K mean, while full \trm\ has the largest MATH-500 and AIME means. Along $w_{\rm M}=w_{\rm A}=(1-w_{\rm G})/2$, full \trm\ leads up to $w_{\rm G}=7.11/7.32\approx0.9713$. Along the equal-GSM8K/MATH-500 slice, its advantage is $4.225+1.335w_{\rm A}$ and remains positive throughout. Equal family weighting lies in the full model\textquotesingle s region. This geometry describes priorities over the same mean score profiles.

\FloatBarrier
\subsection{Question resolution and alternative aggregation}
The mathematics evaluation contains 1,319 GSM8K questions, 500 MATH-500 questions, and 60 pooled AIME questions per seed. For that seed\textquotesingle s correct-answer counts $n_{\rm G},n_{\rm M},n_{\rm A}$,
\begin{equation}
\mathrm{MathAvg}=\frac{100}{3}\left(\frac{n_{\rm G}}{1319}+\frac{n_{\rm M}}{500}+\frac{n_{\rm A}}{60}\right).
\label{eq:atlas-countmetric}
\end{equation}
One extra correct answer in a single evaluation changes that seed\textquotesingle s MathAvg by approximately 0.0253, 0.0667, or 0.5556 points for GSM8K, MATH-500, or pooled AIME. The three-seed mean spreads a change in one seed over three evaluations. Equal-family weighting prevents the largest test set from determining the aggregate solely through its question count.

\begin{figure}[!htb]
\centering
\includegraphics[width=\linewidth]{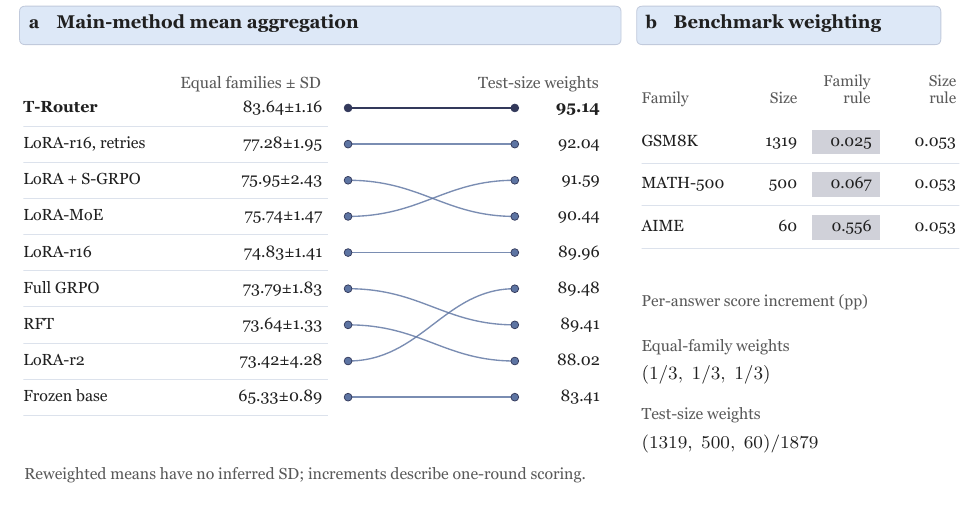}
\caption{\textbf{Aggregation at full-test sizes.} The same nine method mean profiles are summarized by equal-family MathAvg and equal-question reaggregation. The size/resolution panel reports per-evaluation answer increments and family weights. The alternative weighted means are calculated from reported mean scores.}
\label{fig:atlas-resolution}
\end{figure}

An equal-question summary instead uses
\begin{equation}
S_{\rm question}=\frac{1319s_{\rm G}+500s_{\rm M}+60s_{\rm A}}{1879},
\label{eq:atlas-micro}
\end{equation}
with weights approximately $(0.7020,0.2661,0.0319)$. One correct answer changes this per-seed summary by $100/1879\approx0.0532$ points. Reaggregating full \trm's mean profile gives approximately 95.14, compared with its reported equal-family MathAvg of 83.64. Full \trm\ has the highest mean in all three mathematical families among the nine main methods, so it remains first under every nonnegative family weighting summing to one. Both summaries express the same evaluated profile with different emphasis.

\FloatBarrier
\subsection{Evaluation-seed dispersion and complete variance record}\label{app:evaluation-variance}
The source evaluations use seeds 42, 43, and 44. For a reported score or paired difference $x_r$, the sample variance and SD are
\begin{equation}
\bar x=\frac{1}{3}\sum_{r=1}^3x_r,\qquad s^2=\frac{1}{2}\sum_{r=1}^3(x_r-\bar x)^2,\qquad s=\sqrt{s^2}.
\label{eq:evaluation-variance}
\end{equation}
Tables~\ref{tab:variance-math-search}--\ref{tab:variance-historical-auxiliary} preserve all seven supplied sample-variance panels at four decimals. Mathematical accuracy, MathAvg, and their differences use pp$^2$; search scores use the corresponding points$^2$. Aggregates and paired differences are formed within each evaluation seed before their variances are computed. Their supplied variances therefore retain the joint task/seed behavior; averaging component variances would not produce the same statistic. SDs describe evaluation variability for the reported configurations. Resource measurements and original $\ddagger$ references keep their separate measurement status.

\begin{table}[!htb]
\centering\footnotesize
\setlength{\tabcolsep}{3.0pt}
\caption{\textbf{Sample variances: mathematics and search.} All values are supplied $s^2$ over three evaluation seeds. Zero is an exact zero sample variance; $\Delta$ pairs each method with full-parameter GRPO at the same seed index.}
\label{tab:variance-math-search}
\textbf{(a) Mathematics}\\[3pt]
\begin{tabular}{lrrrrrrr}
\toprule
\rowcolor{paperhead}
Method & GSM8K & MATH & AIME24 & AIME25 & AIME mean & MathAvg & $\Delta$ \\
\midrule
Frozen base & 0.8986 & 5.0800 & 33.3333 & 11.1111 & 11.1111 & 0.7874 & 0.8933 \\
Full-parameter GRPO & 0.0690 & 0.4933 & 48.1481 & 14.8148 & 28.7037 & 3.3394 & 0.0000 \\
RFT & 1.3412 & 6.0933 & 103.7037 & 181.4815 & 6.4815 & 1.7760 & 6.4687 \\
LoRA-r2 + GRPO & 0.0536 & 0.6533 & 300.0000 & 114.8148 & 162.0370 & 18.2979 & 7.1083 \\
LoRA-r16 + GRPO & 0.5882 & 6.4133 & 33.3333 & 48.1481 & 39.8148 & 2.0014 & 7.2478 \\
LoRA-MoE + RO-GRPO & 0.3679 & 2.0933 & 25.9259 & 25.9259 & 6.4815 & 2.1701 & 0.1681 \\
LoRA + S-GRPO & 0.2146 & 3.4533 & 381.4815 & 11.1111 & 67.5926 & 5.9142 & 5.2820 \\
LoRA-r16, matched retries & 0.2548 & 0.0133 & 59.2593 & 159.2593 & 36.1111 & 3.8180 & 14.1538 \\
\trm\ + GRPO & 0.6227 & 0.4933 & 44.4444 & 3.7037 & 12.0370 & 1.3391 & 8.6068 \\
\bottomrule
\end{tabular}
\par\medskip
\textbf{(b) Agentic search}\\[3pt]
\begin{tabular}{lrrrr}
\toprule
\rowcolor{paperhead}
Method & BrowseComp & $\Delta$ F1 & ASearch & $\Delta$ score \\
\midrule
Frozen base & 0.0925 & 1.5637 & 0.1699 & 0.3186 \\
Full-parameter GRPO & 2.3867 & 0.0000 & 0.2777 & 0.0000 \\
RFT & 0.5584 & 4.2884 & 0.8434 & 1.3926 \\
LoRA-r2 + GRPO & 0.3094 & 1.5502 & 0.3459 & 1.0129 \\
LoRA-r16 + GRPO & 0.3189 & 3.2539 & 0.0317 & 0.4089 \\
LoRA-MoE + RO-GRPO & 1.3643 & 6.5709 & 0.3075 & 0.4258 \\
LoRA + S-GRPO & 0.3256 & 2.4924 & 0.0371 & 0.1147 \\
LoRA-r16, matched retries & 1.2793 & 1.7436 & 0.5555 & 0.3915 \\
\trm\ + GRPO & 1.3487 & 0.8305 & 0.5536 & 1.6154 \\
\bottomrule
\end{tabular}
\end{table}

\begin{table}[!htb]
\centering\footnotesize
\setlength{\tabcolsep}{3.0pt}
\caption{\textbf{Sample variances: components and writeback.} The component $\Delta$ is each variant minus full \trm. Historical writeback settings retain their $\dagger$ estimate attribute.}
\label{tab:variance-ablation-writeback}
\textbf{(c) Component ablations}\\[3pt]
\begin{tabular}{lrrrrrr}
\toprule
\rowcolor{paperhead}
Configuration & GSM8K & MATH & AIME24 & AIME25 & MathAvg & $\Delta$ \\
\midrule
Full \trm & 0.6227 & 0.4933 & 44.4444 & 3.7037 & 1.3391 & 0.0000 \\
No cross-layer retrieval & 0.3468 & 2.4400 & 14.8148 & 225.9259 & 5.5468 & 3.1758 \\
Hidden-state memory & 0.2836 & 1.7200 & 114.8148 & 25.9259 & 4.7817 & 8.7683 \\
Learned static routing & 0.4675 & 0.6533 & 33.3333 & 44.4444 & 2.9100 & 7.7297 \\
No state (matched MLP) & 0.5595 & 3.6933 & 311.1111 & 103.7037 & 16.7764 & 9.5347 \\
No layer/block identity & 0.1744 & 3.2533 & 114.8148 & 11.1111 & 1.1156 & 0.3102 \\
No RMS alignment & 0.4158 & 14.4533 & 11.1111 & 11.1111 & 0.1699 & 1.5405 \\
No token-conditioned gate & 0.1744 & 3.2933 & 70.3704 & 137.0370 & 6.7967 & 2.1560 \\
No informative retries & 0.2989 & 6.4933 & 3.7037 & 14.8148 & 1.4390 & 5.1010 \\
\bottomrule
\end{tabular}
\par\medskip
\textbf{(d) Writeback configurations}\\[3pt]
\begin{tabular}{lrrrr}
\toprule
\rowcolor{paperhead}
Configuration & GSM8K & MATH & AIME mean & MathAvg \\
\midrule
Standard gate$^\dagger$ & 0.7645 & 10.8133 & 28.7037 & 7.9975 \\
Fixed RMS 0.5\%$^\dagger$ & 0.9791 & 2.5200 & 12.0370 & 2.8583 \\
Fixed RMS 1\%$^\dagger$ & 0.4541 & 9.6533 & 25.0000 & 5.4889 \\
Fixed RMS 2\%$^\dagger$ & 0.3123 & 4.4400 & 6.4815 & 1.3420 \\
Learned RMS, init 0.5\%$^\dagger$ & 0.0824 & 0.9733 & 69.4444 & 5.7185 \\
Learned RMS, init 1\%$^\dagger$ & 0.2088 & 2.7733 & 144.4444 & 19.3792 \\
Full \trm & 0.6227 & 0.4933 & 12.0370 & 1.3391 \\
\bottomrule
\end{tabular}
\end{table}

\begin{table}[!htb]
\centering\footnotesize
\setlength{\tabcolsep}{3.0pt}
\caption{\textbf{Sample variances: historical profiles and auxiliary scores.} Changes in (e) are paired fixed-RMS-minus-original differences. Dashes in (g) preserve unavailable variances for original non-generative references or missing scores.}
\label{tab:variance-historical-auxiliary}
\textbf{(e) Structure--writeback combinations}\\[3pt]
\begin{tabular}{lrrr}
\toprule
\rowcolor{paperhead}
Structure & Original & Fixed RMS 2\% & Change \\
\midrule
Dense, standard gate$^\dagger$ & 7.9975 & 1.3420 & 5.1199 \\
Wide512$^\dagger$ & 8.1360 & 0.4337 & 11.6171 \\
Top2 composite$^\dagger$ & 4.5150 & 2.6144 & 3.8024 \\
Local64$^\dagger$ & 6.1259 & 0.8179 & 2.6340 \\
\bottomrule
\end{tabular}
\par\medskip
\textbf{(f) Other historical configurations}\\[3pt]
\begin{tabular}{lrrrrr}
\toprule
\rowcolor{paperhead}
Configuration & GSM8K & MATH & AIME24 & AIME25 & MathAvg \\
\midrule
Standard gate$^\dagger$ & 0.7645 & 10.8133 & 0.0000 & 114.8148 & 7.9975 \\
Standard gate, no KL$^\dagger$ & 0.3353 & 16.1733 & 33.3333 & 3.7037 & 2.9447 \\
Local16$^\dagger$ & 1.2473 & 1.0000 & 11.1111 & 337.0370 & 7.5142 \\
Local64$^\dagger$ & 0.3276 & 6.6133 & 25.9259 & 25.9259 & 6.1259 \\
Wide512$^\dagger$ & 1.3852 & 3.7200 & 114.8148 & 25.9259 & 8.1360 \\
Zero initialization$^\dagger$ & 0.1226 & 3.2933 & 114.8148 & 44.4444 & 5.2464 \\
Orthogonal initialization$^\dagger$ & 0.0057 & 2.9733 & 25.9259 & 3.7037 & 0.1421 \\
Top2 composite$^\dagger$ & 0.0824 & 1.1200 & 114.8148 & 92.5926 & 4.5150 \\
\bottomrule
\end{tabular}
\par\medskip
\textbf{(g) Auxiliary task profiles}\\[3pt]
\begin{tabular}{lrrrr}
\toprule
\rowcolor{paperhead}
Configuration & MuSR & GPQA-D & IFEval strict@50 & WikiText-2 PPL \\
\midrule
Frozen base & --- & --- & --- & --- \\
Full-parameter GRPO & --- & --- & --- & --- \\
Standard gate & --- & --- & 9.3333 & --- \\
Fixed RMS 2\% & --- & --- & 16.0000 & --- \\
Learned RMS, init 1\% & --- & --- & 52.0000 & --- \\
Full \trm & --- & --- & 17.3333 & --- \\
\bottomrule
\end{tabular}
\end{table}

\clearpage
\section{Design interpretation and functional correspondence}\label{app:design}
\trm\ pursues parameter-efficient reinforcement learning by concentrating adaptation in a communication interface around a frozen backbone. Completed block changes become explicitly accessible to later receivers, which learn their selection and relative influence. Separating transported content, recurrent depth state, and the current control signal gives each part of the trainable budget a concrete role. This is the computational form of the thalamic routing principle.

\subsection{Computation and coordination as separate roles}
The relevant thalamic function is context-dependent regulation of cortical communication. The pulvinar coordinates information transmission between cortical areas through attention-dependent synchronization \citep{saalmann2012pulvinar}. Mediodorsal thalamic input amplifies functional prefrontal connectivity, sustaining rule representations without itself relaying their categorical content \citep{schmitt2017thalamic}. Together, these findings motivate a routing interface that uses task context to regulate how distributed computations influence a receiver.

In \trm, pretrained decoder weights provide the computational substrate, while the complete auxiliary module regulates communication. Compressors define accessible source content, the depth controller conditions routing, and calibrated writeback changes the inputs processed by the frozen layers. Green source storage, purple controller state, and gold residual writeback in Figure~\ref{fig:framework} belong to this one coordinating system. The cortical partitions in Figure~\ref{fig:problem} depict distributed backbone computation; the thalamic correspondence covers the full auxiliary pathway.

\begin{table}[!htbp]
\centering\small
\caption{\textbf{Functional correspondence and its computational realization.} Each row specifies an operation in the design.}
\label{tab:analogy}
\begin{tabularx}{\linewidth}{>{\raggedright\arraybackslash}p{2.9cm}YY}
\toprule
\rowcolor{paperhead}
Organizing principle & T-Router realization & Computational consequence \\
\midrule
Distributed processing & Frozen decoder blocks & Every block produces its original nonlinear transformation on the received residual \\
Selective inter-area communication & Attention over completed block records & Later computations select from origin-indexed intermediate changes \\
Context-dependent coordination & Token-specific slots recurrent over depth & Retrieval uses an auxiliary history of preceding computational states \\
Modulation of influence & Signed RMS-calibrated writeback & Direction and relative magnitude are controlled separately \\
Coordinating system around existing processing & Full source--controller--writeback pathway & Adaptation capacity is concentrated in communication between frozen blocks \\
\bottomrule
\end{tabularx}
\end{table}

This mapping determines concrete design choices. A receiving block retains its full transformation instead of becoming an expert that may be skipped. Each source has an explicit depth origin rather than becoming an anonymous part of an accumulated hidden state. The same coordinate system is used throughout the paper: block index identifies an available computation, slot index identifies a component of the routing context, and token index identifies the causal position at which the decision occurs. The routing correspondence is operational: source weights determine which completed changes contribute to the mixture, and the gate controls their signed influence on the receiving block.

\subsection{Why the source bank and controller carry different information}
The source bank preserves completed displacements, each compressed using a source-specific map. Its size grows with completed depth blocks, and an appended record remains fixed for subsequent reads in that forward call. The controller state $S$ rewrites a fixed number of slots at each layer, accumulating update proposals through decay and slot allocation. The context $P$ is a query-dependent readout of the updated slots, recomputed for each receiving layer.

These objects have complementary roles. The bank supplies content associated with earlier blocks; $S$ summarizes the trajectory of residual states, source summaries, and layer identities. The current residual queries this state to obtain $P$. Routing and gating then combine $P$ with the current residual to select source content and regulate its influence. The transported vector $c$ comes from the source bank. This separation lets accessible sources grow without increasing the number of controller slots.

\begin{table}[!htbp]
\centering\small
\caption{\textbf{Source storage and controller state are complementary interfaces.} Their different update rules organize the computation even when their widths coincide.}
\label{tab:bank-controller-comparison}
\begin{tabularx}{\linewidth}{lYY}
\toprule
\rowcolor{paperhead}
Property & Source bank & Depth controller \\
\midrule
Index meaning & Completed block identity & Slot identity \\
Cardinality & Grows as blocks finish & Fixed at $K=8$ \\
Content width & $r=256$ & $p=256$ per slot \\
Write frequency & Once per completed block & Once per decoder layer \\
Write rule & Append compressed block change & Decay plus allocated shared proposal \\
Read function & Source attention produces $c_\ell$ & Slot attention produces $P_\ell$ \\
Direct contribution & Content projected into the residual & Context for source selection and gate input \\
Lifetime & One forward call & One forward call \\
\bottomrule
\end{tabularx}
\end{table}

Equal dimensions do not make the two representations interchangeable. A bank of seven records at layer 31 retains seven block identities. Eight controller slots at that layer can each contain a different mixture of many earlier proposals. The controller's rank-one increment per layer is compatible with a higher-rank accumulated state, as shown in Appendix~\ref{app:slot-history}. The two representations also have different sensitivities to intervention: changing a source record alters both retrieved content and later controller input, whereas changing only a controller state changes how existing source content is selected and scaled.

\subsection{Three coordinates of a residual intervention}
The raw direction can be written as $w_\ell=U_\ell c_\ell$. This expression contains source selection through the attention-weighted source mixture and transport through the receiving projection. On the active branch, calibration maps the raw direction onto a sphere whose radius is the receiving residual norm. Finally, the signed gate chooses the intervention's relative magnitude and orientation along that calibrated direction.

The three coordinates have distinct mathematical roles. Source attention is a distribution over candidate origins; it does not specify the norm of their projected sum. The projection supplies a receiver-specific low-dimensional subspace; a scalar gate does not change that subspace. Calibration changes the norm of a nonzero direction while preserving its ray, and the gate then selects a point along the line spanned by that direction. On the active branch, $|g_\ell|$ has the direct interpretation $\|R_\ell\|/\|h_\ell\|$.

\begin{table}[!htbp]
\centering\small
\caption{\textbf{Coordinates of controlled reuse.} Distinct parameter groups control the origin, direction, and relative strength of the intervention.}
\label{tab:intervention-coordinates}
\begin{tabularx}{\linewidth}{lYY}
\toprule
\rowcolor{paperhead}
Coordinate & Governing objects & Quantity it controls \\
\midrule
Source preference & Query, source keys, depth context & Relative attention across completed computations \\
Content transport & $C_b$, source value projection, $U_\ell$ & Source-to-receiver transformation within $\operatorname{col}(U_\ell)$ \\
Relative influence & Target-RMS calibration and signed gate & Norm relative to the receiving residual and sign along the direction \\
\bottomrule
\end{tabularx}
\end{table}

This factorization explains why writeback calibration and routing structure interact. Sparsifying source selection changes the mixture delivered to the receiving projection. Normalizing stored deltas changes the geometry on which keys and values are built. Changing controller width changes the context available to the query and gate. A common fixed writeback coefficient can therefore act on different learned directions in different designs. The paired results in Appendix~\ref{app:variants} reveal this interaction at the configuration level. The matched component study in Table~\ref{tab:ablation-complete} separately tests the operations of the full interface.

The learned gate adds another degree of freedom to calibrated transport. Its input contains both the current residual and controller context, so the model can modulate an already selected direction differently across tokens and receiving layers. A fixed gate retains token-dependent directions through source routing, but uses a common scalar coefficient. Thus ``fixed'' and ``learned'' refer to the strength coordinate; they do not divide the models into static and dynamic source-selection systems.

\subsection{Relation to alternative adaptation interfaces}
A local low-rank weight update changes a layer matrix by $\Delta W=AB$ and acts on the current input to that matrix \citep{hu2022lora}. A bottleneck adapter applies a learned local transformation around a layer \citep{houlsby2019adapters}. \trm\ instead retains explicit earlier sources and chooses a combination of them at a receiving layer. Its transport matrices are also low rank, but their operands are completed block changes selected through an input-dependent depth route.

DenseNet exposes earlier layer features through fixed feature concatenation \citep{huang2017densenet}. \trm\ reads compressed block changes through learned attention, projects the mixture into the receiving residual space, and calibrates its influence. This gives a fixed-width frozen backbone explicit access to historical computation, with reuse conditioned on the current residual and depth context.

The distinction can be expressed at a fixed receiver. With fixed source weights and no recurrent dependence, the change-dependent part of the raw writeback is a weighted sum of $T_{\ell b}D_b$ terms, with $\operatorname{rank}(T_{\ell b})\le r$. With the full module, the weights, records, controller context, calibration factor, and gate all vary with the input. A low-dimensional intervention space therefore coexists with a nonlinear input-dependent adaptation function. This is different from treating low rank as a statement that the entire adapted network is a fixed linear perturbation.

Delta Attention Residuals uses sublayer or block changes as routed content \citep{luo2026delta}; mHC-based finetuning learns read/write routing over multiple residual streams around frozen branches \citep{oldenburg2026mhcpeft}. \trm's coupled interface assigns separate roles to compressed origin-indexed records, recurrent context, and relative-scale writeback. The bank preserves completed changes, while controller slots accumulate depth history. A current-state query reads those slots to obtain $P$, which jointly conditions source attention and the signed gate. This organization makes content reuse and its control trainable within one compact RL interface.

Reuse operates during the causal forward pass: a receiver draws on completed records to shape the input to its frozen transformation. During RL, the task objective assigns credit through the resulting computation, differentiating the source, controller, routing, and writeback parameters. Frozen layers retain their input derivatives, connecting these communication choices to the output loss. Training therefore learns a reuse rule that executes directly during generation.

\subsection{Connecting the empirical comparisons to the interface}
The component study connects the organization of adaptation capacity to reasoning quality. Block-change memory reaches 83.64 MathAvg versus 67.46 for hidden-state memory at the same allocation; removing retrieval gives 69.48. The full controller also exceeds a capacity-matched MLP, which has 1,646 more parameters and scores 72.80. These comparisons favor accessible changes and recurrent control within the same compact-budget regime. $S$ accumulates depth history, while the current hidden state queries $S$ to obtain $P$ as the context for reuse.

Learned static routing gives 64.98, and removing the token-conditioned gate gives 63.78. The two variants address distinct receiving-layer decisions: which earlier changes contribute and how strongly their mixture enters the residual. All backbone layers continue to execute in order. The supplementary width, initialization, local-branch, and sparse-composite studies explore coupled design choices around this interface; their interaction patterns complement the component comparisons.

The resulting optimization target is a rule for composing existing intermediate contributions. The bank defines available content, depth context conditions its selection, and writeback regulates its influence on the next frozen transformation. Training these choices together yields 83.64 MathAvg with a 41.73M-parameter allocation, compared with 73.79 for full-parameter GRPO and 77.28 for matched-retry LoRA. Computation reuse thus connects the efficiency objective to both an explicit mechanism and its measured reasoning performance.

\end{document}